\def\year{2019}\relax

\documentclass[letterpaper]{article} 
\usepackage{aaai19}
\usepackage{times}  
\usepackage{helvet}  
\usepackage{courier}  
\usepackage[hyphens]{url}  
\usepackage{graphicx}  
\usepackage{epsfig} 
\usepackage{amsmath} 
\usepackage{amssymb}  
\usepackage{amsthm}
\usepackage{thmtools,thm-restate}
\usepackage{newtxtext}
\usepackage{xspace}
\usepackage{bbm} 
\usepackage{color}
\usepackage{amsfonts}
\usepackage{mathtools}
\usepackage{enumitem}
\usepackage[ruled, linesnumbered]{algorithm2e}
\usepackage[noend]{algpseudocode}
\usepackage{url}
\usepackage{caption}
\usepackage{subcaption}
\usepackage{bm}
\usepackage{multirow}
\usepackage{booktabs}
\usepackage{tabulary}
\usepackage{appendix}
\usepackage{xspace}
\usepackage{environ}
\usepackage{tabulary}
\usepackage{lipsum}
\usepackage{flushend}

\newcounter{quote}

\NewEnviron{myquote}{
	\vspace{1ex}\par
	\refstepcounter{quote}%
	\hfill
	\parbox{\dimexpr \linewidth-1cm}%
	{\emph{\BODY}}%
	\hfill
	\vspace{1ex}
	(\thequote)\par
	}

\renewenvironment{proof}[1]{{\large \bfseries Proof #1}}{\qed}
\newcommand{\hypref}[1]{\textbf{H~\ref{#1}}}

\graphicspath{{./figs/}}

\newboolean{CONF}
\newboolean{ARXIV}

\setboolean{CONF}{false}
\ifthenelse{\boolean{CONF}}
	{\setboolean{ARXIV}{false}}
	{\setboolean{ARXIV}{true}}

\newcommand{\textVersion}[2]
{\ifthenelse{\boolean{CONF} }{#1}{}\ifthenelse{\boolean{ARXIV}}{#2}{}}

\newboolean{ShowAuthors}
\setboolean{ShowAuthors}{false}
\def\blind#1{
\ifthenelse{\boolean {ShowAuthors}}{#1}{}
}

\newcommand{\vs}{\ensuremath{v_{\rm s}}}
\newcommand{\vg}{\ensuremath{v_{\rm t}}}

\newcommand{\graph}{\ensuremath{\calG}}

\newcommand{\edgeSet}{\ensuremath{\calE}}
\newcommand{\vertexSet}{\ensuremath{\calV}}
\newcommand{\vertex}[0]{\ensuremath v}
\newcommand{\edge}[0]{\ensuremath e}

\newcommand{\vertu}[0]{\ensuremath u}

\newcommand{\Path}[0]{\ensuremath \sigma \xspace}
\newcommand{\weight}[0]{\ensuremath w}

\newcommand{\Eeval}{\ensuremath{\calE_{\rm eval}}\xspace}
\newcommand{\Vexp}{\ensuremath{\calV_{\rm rwr}}\xspace}

\newcommand{\costEval}[0]{\ensuremath c_e}
\newcommand{\costExp}[0]{\ensuremath c_r}
\newcommand{\costTotal}[0]{\ensuremath C}

\newcommand\algname[1]{\texttt{#1}\xspace}
\newcommand\lwastar{\algname{LWA*}}
\newcommand\astar{\algname{A*}}
\newcommand\lazySP{\algname{LazySP}}
\newcommand\lazyPRM{\algname{LazyPRM}}
\newcommand{\lrastar}{\algname{LRA*}}
\newcommand{\glrastar}{\algname{GLS}}

\newcommand{\world}[0]{\ensuremath \phi}
\newcommand{\ALG}[0]{\textsc{Alg}\xspace}

\newcommand{\Event}[0]{\textsc{Event}\xspace}
\newcommand{\Selector}[0]{\textsc{Selector}\xspace}

\newcommand{\eventShortestPath}[0]{\textsc{ShortestPath}\xspace}
\newcommand{\eventPathExistence}[0]{\textsc{SubpathExistence}\xspace}
\newcommand{\eventConstantDepth}[0]{\textsc{ConstantDepth}\xspace}
\newcommand{\eventHeuristicProgress}[0]{\textsc{HeuristicProgress}\xspace}

\newcommand{\selectorForward}[0]{\textsc{Forward}\xspace}

\newcommand{\selectorAlternate}[0]{\textsc{Alternate}\xspace}
\newcommand{\selectorGreedy}[0]{\textsc{FailFast}\xspace}

\newcommand{\hmin}{h_{\text{min}}}
\newcommand{\pmax}{p_{\text{max}}}

\newcommand{\EevalLSP}{\ensuremath{\calE_{\rm eval, LSP}}\xspace}
\newcommand{\VexpLSP}{\ensuremath{\calV_{\rm rwr, LSP}}\xspace}

\newcommand{\EevalGLS}{\ensuremath{\calE_{\rm eval, GLS}}\xspace}
\newcommand{\VexpGLS}{\ensuremath{\calV_{\rm rwr, GLS}}\xspace}

\newcommand{\LazyTree}{\ensuremath{\calT_{\rm lazy}}\xspace}
\newcommand{\Subpath}{\ensuremath{\Path_{\rm sub}}\xspace}

\newcommand{\glsspf}{\algname{SP + F}}
\newcommand{\glsspa}{\algname{SP + A}}
\newcommand{\glsspff}{\algname{SP + FF}}

\newcommand{\glscdf}{\algname{CD + F}}
\newcommand{\glscda}{\algname{CD + A}}
\newcommand{\glscdff}{\algname{CD + FF}}

\newcommand{\glssef}{\algname{SE + F}}
\newcommand{\glssea}{\algname{SE + A}}
\newcommand{\glsseff}{\algname{SE + FF}}
\DeclareMathOperator*{\argmin}{arg\,min}

\newcommand{\argminprob}[1]{\underset{#1}{\argmin}}

\DeclareMathOperator*{\suchthat}{\;\; \mbox{s.t.} \;\;}
\newcommand{\abs}[1]{\left|#1 \right|}
\newcommand{\pair}[2]{\left( #1, #2\right)}

\newcommand{\seq}[2]{\left(#1_{1}, #1_{2}, \ldots, #1_{#2}\right)}

\newcommand{\expect}[2]{\mathbb{E}_{#1}\left[#2\right]}

\newcommand{\calE}{\ensuremath{\mathcal{E}}\xspace}
\newcommand{\calG}{\ensuremath{\mathcal{G}}\xspace}

\newcommand{\calT}{\ensuremath{\mathcal{T}}\xspace}
\newcommand{\calV}{\ensuremath{\mathcal{V}}\xspace}

\newcommand{\R}{\mathbb{R}} 

\newtheorem{thm}{Theorem}[section]

\newtheorem{hypothesis}{H}

\newtheorem{cor}{Corollary}[section]

\newtheorem{definition}{Definition}[section]

\newcommand{\ignore}[1]{}

\newcommand{\vectorp}{\mathbf{p}}

\definecolor{orange}{rgb}{1,0.5,0}

\makeatletter 
\g@addto@macro{\@algocf@init}{\SetKwInOut{Parameter}{Parameters}} 
\makeatother
 
\begin{document}
%
\title{\huge{Generalized Lazy Search for Robot Motion Planning: \\ Interleaving Search and Edge Evaluation via Event-based Toggles}}
\author{Aditya Mandalika \\ University of Washington \\ \texttt{adityavk@cs.uw.edu}
\thanks{
This work was (partially) funded by the National Institute of Health R01 (\#R01EB019335), National Science Foundation CPS (\#1544797), National Science Foundation NRI (\#1637748), the Office of Naval Research, the RCTA, Amazon, and Honda.} 
\And
Sanjiban Choudhury \\ University of Washington \\ \texttt{sanjibac@cs.uw.edu}
\footnotemark[1]
\And 
Oren Salzman \\ Carnegie Mellon University \\ \texttt{osalzman@andrew.cmu.edu}
\footnotemark[1]
\And
Siddhartha Srinivasa \\ University of Washington \\ \texttt{siddh@cs.uw.edu}
\footnotemark[1]}
\maketitle
\begin{abstract}
Lazy search algorithms can efficiently solve problems where edge evaluation is the bottleneck in computation, as is the case for robotic motion planning.
The optimal algorithm in this class, \lazySP, lazily restricts edge evaluation to only the shortest path. 
Doing so comes at the expense of search effort, i.e., \lazySP must \emph{recompute the search tree} every time an edge is found to be invalid.  
This becomes prohibitively expensive when dealing with large graphs or highly cluttered environments. 
Our key insight is the need to balance both edge evaluation and search effort to minimize the total planning time. 
Our contribution is two-fold. 
First, we propose a framework, Generalized Lazy Search (\glrastar), that seamlessly toggles between search and evaluation to prevent wasted efforts.   
We show that for a choice of toggle, \glrastar is provably more efficient than \lazySP.
Second, we leverage prior experience of edge probabilities to derive \glrastar policies that minimize expected planning time.
We show that \glrastar equipped with such priors significantly outperforms competitive baselines for many simulated environments in $\mathbb{R}^2, SE(2)$ and 7-DoF manipulation.
\end{abstract}

\section{Introduction}
\label{sec:introduction}

We focus on the problem of finding the shortest path on a graph while minimizing total planning time. 
This is critical in applications such as robotic motion planning~\cite{L06}, where collision-free paths must be computed in real time.
A typical search algorithm expands a wavefront from the start, evaluating edges discovered until it finds the shortest feasible path to the goal.
The planning time then becomes the sum of the time spent in two phases -- \emph{search effort} and \emph{edge evaluation}. 
While edge evaluation is generally more expensive in motion planning~\cite{hauser15lazy}, the \emph{actual ratio of these times varies} with problem instances and graph sizes. 
Our goal is to design a framework of algorithms that let us balance this trade-off. 

\begin{figure}[!tb]
\centering
\includegraphics[width=\columnwidth]{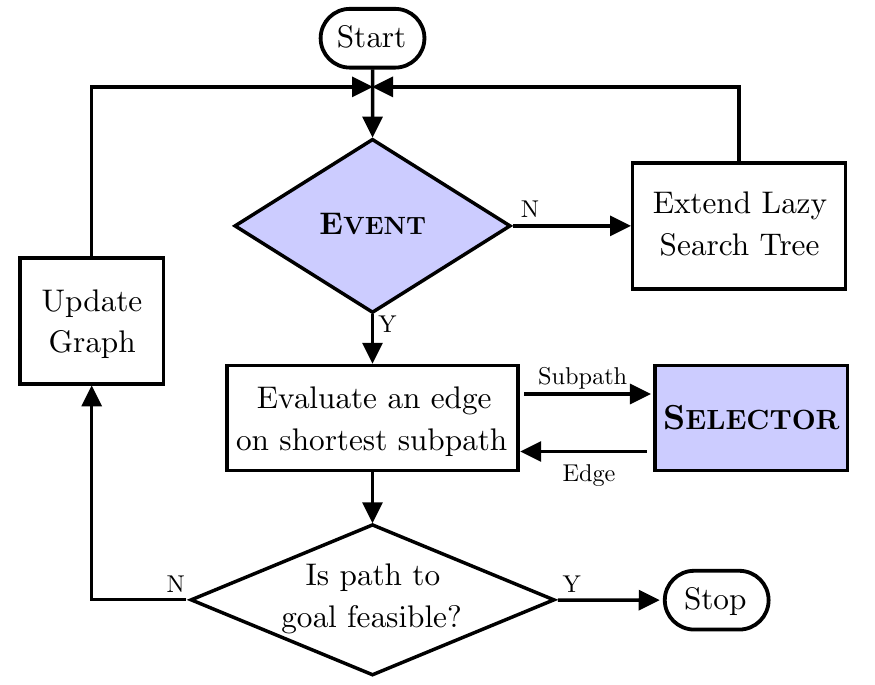}
\caption{The Generalized Lazy Search (\glrastar) framework with two parameters - \Event and \Selector (blue)}
\label{fig:intro}
\end{figure}

Unfortunately, current shortest path algorithms do not provide a framework flexible enough to traverse the pareto curve between search effort and edge evaluation.
On one end of the spectrum, \astar{} and its variants~\cite{HNR68,partialA*,korf1985} evaluate edges as soon as they are discovered. 
Hence although \astar{} is optimal in terms of \emph{search effort}, it is at the cost of excessive edge evaluations.
On the other hand, \lazySP{}~\cite{DS16} amongst other lazy search techniques~\cite{lazyPRM,CPL14,hauser15lazy}, expands the search wavefront all the way to the goal before evaluating edges. 
Hence \lazySP{} is optimal in terms of edge evaluation but has to replan everytime an edge is invalidated. 

\begin{figure*}[!t]
    \centering
    \includegraphics[width=\textwidth]{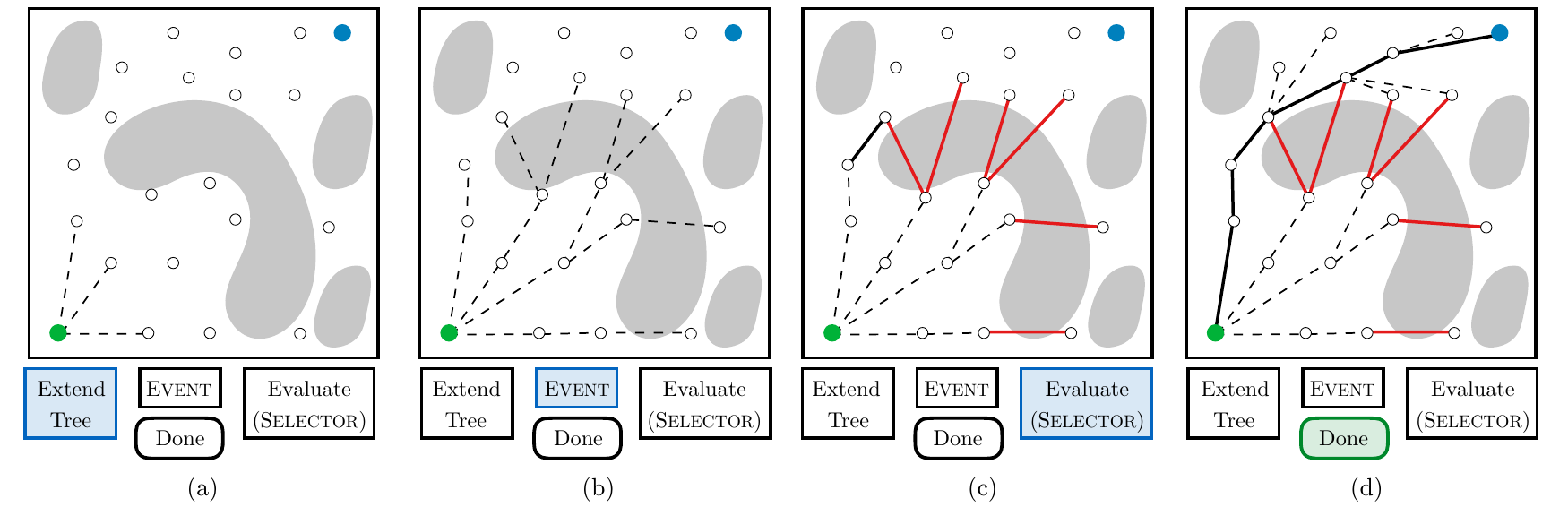}
    \caption{Mechanics of the \glrastar framework (Algorithm~\ref{alg:GLS}) for an ideal \Event and \Selector combination.}
  \label{fig:illustration_algorithm}
\end{figure*}%


In this work, we propose a framework for \emph{algorithmically toggling} between search effort and edge evaluation.  We are guaranteed to find the shortest path as long as the following holds true; the search tree must always be repaired to be consistent, and edge evaluation must be restricted to the shortest subpath in the tree. 
 Our framework, \emph{Generalized Lazy Search} (\glrastar), has two modules - \Event and \Selector (Fig.~\ref{fig:intro}). The algorithm expands a lazy search tree without evaluating any edges till the \Event is triggered. A \Selector is then invoked to evaluate an edge on the shortest subpath in the lazily expanded search tree. We show that by choosing different \Event and \Selector pairs, we can recover several existing lazy search algorithms such as \lazySP~\cite{DS16}, \lwastar~\cite{CPL14} and \lrastar~\cite{Mandalika18}. 

What constitutes an optimal trade-off and can this be captured by \glrastar?
Consider the ideal scenario, one with an omniscient oracle~\cite{haghtalab2017provable} that knows ahead of time which edges are valid or invalid.   
In fact, the oracle can compute the minimal set of invalid edges $\mathcal{I}$ that must be invalidated to arrive at the shortest feasible path.
How can we utilize such an oracle in \glrastar?
A simple strategy is as follows; as the search wavefront expands from start to goal, the oracle monitors the new edges that are discovered and triggers an \Event if it belongs to  $\mathcal{I}$. A \Selector then evaluates that edge. 
This minimizes edge evaluation and curtails wasted search effort. 

This insight extends to the more practical setting where we have \emph{priors on edge validity} that are learned from experience. 
We derive \Event and \Selector that minimize the expected planning time. 
This produces behaviors similar to the omniscient oracle (Fig.~\ref{fig:illustration_algorithm}); the search proceeds until the \Event is triggered due to the appearance of low probability edges on the current subpath; the \Selector then selects these edges to invalidate the subpath; and the process continues until the shortest feasible path is found. 


We make the following contributions:
\begin{enumerate}
\item We propose a class of algorithms, \glrastar (Section~\ref{sec:framework}), that minimize computational effort, defined as a function of both edge evaluation and vertex rewiring (Section~\ref{sec:problem_formulation}). 
\item We recover different lazy search algorithms as instantiations of \glrastar. We further prove that one such instantiation is edge optimal and causes fewer rewires than \lazySP (Section~\ref{sec:framework}, Theorem \ref{theorem:optimal}).
\item We derive instantiations of \glrastar that exploit the availability of edge priors to minimize expected computational effort (Section~\ref{sec:approach}, Theorem~\ref{theorem:greedy_selector}).
\item We show that \glrastar informed with edge priors can outperform competitive baselines on a spectrum of planning domains (Section~\ref{sec:experiments}).
\end{enumerate}

\section{Related Work}
\label{sec:related_work}

Graphs lend powerful tractability to robotic motion planning~\cite{L06}. They can be explicit, i.e., constructed as part of a pre-processing stage~\cite{kavraki96prm,KF11,JSCP15}, or implicit, i.e., discovered incrementally during search~\cite{likhachev2004ara,GSB15,SH15}. 

\astar~\cite{HNR68} and its variants have enjoyed widespread success in finding the shortest path with an optimal number of vertex expansions. However, in domains where edge evaluations are expensive and dominate the planning time, a \emph{lazy approach} is often employed~\cite{lazyPRM,hauser15lazy,kim2018adaptive}. In this approach, the graph is constructed \emph{without} testing if edges are collision-free. Only a subset of edges are evaluated to save computation time. \lazySP~\cite{DS16} extends the graph up to the goal before checking edges. \lwastar~\cite{CPL14} extends the graph a single step before evaluation. \lrastar~\cite{Mandalika18} trades off these approaches, allowing the search to proceed to an arbitrary lookahead. We generalize this further by introducing an event-based toggle. 

Several works have explored the use of priors in search. FuzzyPRM~\cite{nielsen2000two} evaluates paths that minimize the probability of collision. The Anytime Edge Evaluation (AEE*) framework~\cite{narayanan2017heuristic} uses an anytime strategy for edge evaluation informed by priors. POMP~\cite{CDS16} defines surrogate objectives using priors to improve anytime planning. BISECT~\cite{SC17} and DIRECT~\cite{SC18} cast search as Bayesian active learning to derive edge evaluation. E-graphs~\cite{Phillips-RSS-12} uses priors in heuristics. We focus on using priors to find the shortest path while minimizing expected planning time.  

Several alternate approaches speed up planning by creating efficient data structures~\cite{BF16}, modeling belief over the configuration space~\cite{huh2016learning}, sampling vertices in promising regions~\cite{bialkowski2013free,burns2005sampling} or using specialized hardware \cite{murray2016robot}. Other approaches forego optimality and computing near-optimal paths~\cite{SH16,SPARS}. Our work also draws inspiration from approaches that interleave planning and execution, such as \textsf{LRTA*}~\cite{KORF1990189} and \textsf{LSS-LRTA*}~\cite{koenig2009}.

\section{Problem Formulation}
\label{sec:problem_formulation}

Our goal is to design an algorithm that can solve the Single Source Shortest Path (SSSP) problem while minimizing computational effort. We begin with the SSSP problem.
Let $\graph = \pair{\vertexSet}{\edgeSet}$ be a graph, where $\vertexSet$ denotes the set of vertices and $\edgeSet$ the set of edges. Given a pair of source and target vertices $\pair{\vs}{\vg} \in \vertexSet$, a path $\Path$ is represented as a sequence of vertices $\seq{\vertex}{l}$ such that $\vertex_1 = \vs, \vertex_l = \vg, \forall i,~\pair{\vertex_i}{\vertex_{i+1}} \in \edgeSet$. 
We define a \emph{world} $\world: \edgeSet \rightarrow \{0,1\}$ as a mapping from edges to valid ($1$) or invalid ($0$). A path is said to be \emph{feasible} if all edges are valid, i.e., $\forall \edge \in \Path, \world(\edge) = 1$. 
Let $\weight: \edgeSet \rightarrow \R^+$ be the length of an edge. The length of a path is the sum of edge costs, i.e., $\weight(\Path) = \sum_{\edge \in \Path} \weight(\edge)$. The objective of the SSSP problem is to find the shortest feasible path:
\begin{equation}
	\min_{\Path} \; \weight(\Path) \suchthat{ \forall \edge\in\Path, \world(\edge) = 1 }
\end{equation}


Given an SSSP, we define a shortest path algorithm $\ALG(\graph, \vs, \vg, \world)$ that takes as input the graph $\graph$, the source-target pair $\pair{\vs}{\vg}$, and the underlying world $\world$. The algorithm typically solves the problem by building, verifying and rewiring a shortest path tree from source to target. 

Maintaining the search tree and verifying the shortest feasible path are primarily characterized by two atomic operations: edge evaluation and vertex rewiring.

\begin{definition}[Edge Evaluation]
The operation of querying the world $\world(\edge)$ to check if an edge $\edge$ is valid.
\end{definition}

\begin{definition}[Vertex Rewiring]
The operation of finding and assigning a new parent for a vertex $\vertu$ when an invalid edge is discovered.
\end{definition}

The algorithm returns three terms, i.e, $\Path^*,~\Eeval,~\Vexp = \ALG(\graph, \vs, \vg, \world)$. Here, $\Path^*$ is the shortest feasible path, $\Eeval$ is the \emph{set} of edges evaluated during the search, and $\Vexp$ is the \emph{multiset}\footnote{$\Vexp$ is a multiset since a vertex can potentially be rewired multiple times during the planning cycle.} of vertices rewired. \ALG ensures the following certificate:
\begin{enumerate}
	\item Returned path $\Path^*$ is verified to be feasible, i.e., $\forall \edge \in \Path^*,~\edge \in \Eeval,~\world(\edge) = 1$
	\item All paths shorter than $\Path^*$ are verified to be infeasible, i.e., $\forall \Path_i,~\weight(\Path_i) \leq \weight(\Path^*),~\exists \edge \in \Path_i,~\edge \in \Eeval,~\world(\edge) = 0 $
\end{enumerate}

We now define the computational cost (planning time), of solving the SSSP problem as a function of $\Vexp$ and $\Eeval$. Let $\costEval$ be the average cost of evaluating an edge, and $\costExp$ be the average cost of rewiring a vertex. We approximate the total planning time as a linear combination:
\begin{equation}
\label{eq: computation_cost_definition}
\costTotal(\Eeval, \Vexp) = \costEval \abs{\Eeval} + \costExp \abs{\Vexp}
\end{equation}

Our motivation for defining the cost will become clearer in the following section, where we propose a general framework for $\ALG$. This framework lets us explicitly reason about the terms $\Eeval$ and $\Vexp$ in order to balance them. 



\section{Generalized Lazy Search}
\label{sec:framework}

We propose a framework, Generalized Lazy Search (\glrastar), to solve the problem defined in Section~\ref{sec:problem_formulation}. 
The general concept idea is to \emph{toggle} between lazily searching to a horizon and evaluating edges along the current estimated shortest path. 
This toggle must be chosen appropriately to balance the competing computational costs of edge evaluation and vertex rewiring.



\subsection{The Algorithm}


Algorithm~\ref{alg:GLS} describes the \glrastar framework for the shortest path algorithm $\ALG(\graph, \vs, \vg, \world)$ referred to in Section~\ref{sec:problem_formulation}. This framework requires two functions: \Event{} and \Selector{}.

To solve the SSSP problem, we maintain a shortest path search tree over \graph{}. 
We assume that every call to \world{}, which populates \Eeval{}, is expensive. 
Therefore, we initially assume that all edges in \graph{} are valid and maintain this search tree \emph{lazily}. 
Our algorithm initializes the search tree \LazyTree{} rooted at \vs{} (Line 1).
It begins by iteratively extending \LazyTree into \graph{} (Line 4). The search is guided with an admissible heuristic $h(\vertex,\vg)$. 

The procedure \textsc{ExtendTree} additionally takes as input a function \Event{}. 
Extending \LazyTree{} triggers the \Event{} by definition. 
The algorithm, at this point, discontinues the extension of \LazyTree{} and switches to validate the already constructed search tree. 
Therefore, the \Event{} acts as a toggle between lazy seach and edge evaluation.

\begin{algorithm}[tb]
\SetAlgoLined
\caption{Generalized Lazy Search}
\label{alg:GLS}
\SetKwInOut{Input}{Input}
\SetKwInOut{Parameter}{Parameter}
\SetKwInOut{Output}{Output}
\Input{$\text{Graph } \graph{}, \text{ source } \vs{}, \text{ target } \vg{}, \text{ world } \world{}$}
\Parameter{$\Event,\Selector$}
\Output{$\Path^*,~\Eeval{},~\Vexp$}
\vspace{2mm}
$\Eeval \gets \emptyset, \Vexp \gets \emptyset$ \\
$\LazyTree \gets \{\vs\}$ \Comment{Initialize} \\
\vspace{1mm}
\Repeat{shortest feasible path found $\text{s.t.} \forall \edge \in \Path^*, \world(\edge)=1$}
{
$\LazyTree \gets $ \textsc{ExtendTree~(\Event, \LazyTree)} \Comment{Add \Vexp{}} \\
\vspace{0.5mm}
$\Subpath \gets $ \textsc{GetShortestPathToLeaf~(\LazyTree)} \\
\vspace{0.5mm}
\textsc{EvaluateEdge~(\Selector, \Subpath)}  \Comment{Add \Eeval{}} \\
}
\end{algorithm}

\begin{definition}[\Event]
A function that defines the toggle between extending the lazy search tree and validating it.
\label{definition:Event}
\end{definition}



To solve the SSSP problem and validate \LazyTree{}, the algorithm picks the path, \Subpath{}, to a leaf vertex with the lowest estimated total cost to reach the goal (Line 5). It then evaluates an edge along \Subpath{} to validate the search tree (Line 6). In addition to \Subpath{}, the procedure \textsc{EvaluateEdge} also takes as input a function \Selector{}. The \Selector{} acts on \Subpath{} and returns an edge belonging to it that the algorithm evaluates.

\begin{definition}[\Selector]
A function that defines the strategy to select an edge along a subpath to evaluate.
\label{defintion:Selector}
\end{definition}

Edge evaluation is followed by the extension of \LazyTree{} until the \Event{} is triggered again. If the edge were invalid, the subtree emanating from the edge has to be rewired. We can do this efficiently using the mechanics of LPA*~\cite{KLF04}.

This process of interleaving search with edge evaluation continues until the algorithm terminates with the shortest feasible path from source to goal, if one exists. While the algorithm is guaranteed to return the shortest path, the framework permits the design of \Event and \Selector to reduce the total computation cost of solving the SSSP problem.

\begin{algorithm}[!tbh]
\SetAlgoLined
\caption{Candidate \Event{} Definitions}
\label{alg:event_definitions}

\SetKwInOut{Input}{Input}
\SetKwInOut{Output}{Output}
\SetKwProg{function}{Function}{}{}

\SetKwFunction{shortestPath}{\textsc{ShortestPath}}
\SetKwFunction{pathExistence}{\textsc{SubpathExistence}}
\SetKwFunction{constantDepth}{\textsc{ConstantDepth}}
\SetKwFunction{heuristicProgress}{\textsc{HeuristicProgress}}

$v \gets$ leaf vertex in \LazyTree{} with least estimated cost to \vg
\vspace{2mm}

\function{\shortestPath{}}
{
	\If{$v = \vg$}
	{
		\KwRet\ \textbf{true};
	}
}

\function{\constantDepth{depth $\alpha$}}
{
	$\Subpath \gets$ path from \vs~to $v$ \\
	$\alpha_v \gets$ number of unevaluated edges in $\Subpath$ \\
	\If{$\alpha_v = \alpha$ \textbf{or} $v = \vg$}
	{
		\KwRet\ \textbf{true};
	}
}
\function{\heuristicProgress}
{	
	$\hmin \gets \min_{(u',v') \in \Eeval} h(v', \vg)$\\
	\If{$h(\vertex, \vg) < \hmin$ \textbf{or} $v = \vg$}
	{
		\KwRet\ \textbf{true};
	}
  \KwRet\;
}

\function{\pathExistence{probability $\delta$}}
{
	$\Subpath \gets$ path from \vs~to $v$ \\
	$p \gets \prod_{e\in\sigma}{\vectorp{}(e)}$ \\
	\If{$p \leq \delta$ \textbf{or} $v = \vg$}
	{
		\KwRet\ \textbf{true};
	}
}
\end{algorithm}
 
\begin{algorithm}[!tbh]
\SetAlgoLined
\caption{Candidate \Selector{} Definitions}
\label{alg:selector_definitions}

\SetKwInOut{Input}{Input}
\SetKwInOut{Output}{Output}
\SetKwProg{function}{Function}{}{}

\SetKwFunction{forward}{\textsc{Forward}}
\SetKwFunction{alternate}{\textsc{Alternate}}
\SetKwFunction{greedy}{\textsc{FailFast}}
\SetKwFunction{weighted}{\textsc{Weighted}}

\function{\forward{}}
{
  \KwRet\ \{first unevaluated edge closest to \vs\};
}
\function{\alternate{}}
{
  \If{Iteration Number is Odd}
  {
    \KwRet\ \{first unevaluated edge closest to \vs\};
  }
  \Else
  {
    \KwRet\ \{first unevaluated edge closest to \vg\};
  }
}
\function{\greedy{}}
{
 \KwRet\ \{$\argminprob{\edge~\in \Subpath}~\vectorp(\edge)$\};
}
\end{algorithm}




\subsection{Role of \Event and \Selector}
\label{sec:framework:event_selector}

Since the lazy search paradigm operates based on the concept of optimism under uncertainty, the search tree is extended assuming edges are collision free. However, extending the search tree beyond edges that are in collision can waste computational effort. The \Event acts as a toggle to halt a search deemed wasteful. The \Selector aims to quickly invalidate the path. Fig.~\ref{fig:illustration_algorithm} illustrates the ideal behavior of such an algorithm. Interestingly, the framework can also capture existing lazy search algorithms as different combinations of event and selectors, as shown in Table.~\ref{tab:equivalence}.

\paragraph{\Event.}
 When triggered, events must ensure that the shortest subpath \Subpath{} in \LazyTree{} has at least one unevaluated edge (Theorem~\ref{theorem:complete}). Algorithm~\ref{alg:event_definitions} defines some candidate events. 

\eventShortestPath{} (\texttt{SP}) is triggered when a shortest path to \vg{} has been determined during the lazy extension of \LazyTree. Therefore, in every iteration, this \Event{} presents the \Selector{} with the candidate shortest path from \vs{} to \vg{} on \graph{}. Note that \eventShortestPath{} exhibits algorithmic behavior similar to \lazySP{} and \lazyPRM{}.

\eventConstantDepth{} (\texttt{CD}) is triggered when the procedure \textsc{ExtendTree} chooses to extend a leaf vertex $v \in \LazyTree{}$ such that the subpath from \vs to $v$ has exactly $\alpha$ number of unevaluated edges. Therefore, in every iteration, this \Event{} presents the \Selector{} with \Subpath{} that is characterized by a constant number of unevaluated edges.

\eventHeuristicProgress (\texttt{HP}) is triggered whenever the search expands a vertex whose heuristic value is lower than any vertex whose incident edge has been evaluated. It does so by recording the minimum heurisitic value of a vertex with a parent that has been evaluated, i.e., $\hmin \gets \min_{(u',v') \in \Eeval} h(v', \vg)$. The event is triggered whenever \textsc{ExtendTree} chooses to extend a leaf vertex $v \in \LazyTree{}$ with a heuristic value smaller than $\hmin$. 





\paragraph{\Selector.}
Selectors must ensure that they select at least one unevaluated edge (Theorem~\ref{theorem:complete}).  Algorithm~\ref{alg:selector_definitions} defines some candidate selectors.

Given \Subpath, \selectorForward{} (\texttt{F}) evaluates the first unevaluated edge on \Subpath{} that is closest to \vs. Given a forward search, this constitutes one of the most natural \Selector{}s available. 
\selectorAlternate{} (\texttt{A}) toggles between evaluating the first unevaluated edge closest to \vs{} and \vg{} in every iteration. This approach is motivated by bi-directional search algorithms. Both \Selector{}s were first used in \cite{DS16}.


\begin{table}[!h]
\small
\centering
\begin{tabulary}{\columnwidth}{LCC}\toprule
  {\bf Algorithm}                    				& \Event   						& \Selector  \\ \midrule
  \emph{LazyPRM}~\shortcite{lazyPRM}           		    & \eventShortestPath   & Any     \\
  \lazySP~\shortcite{DS16}           					& \eventShortestPath   & Any     \\
  \lwastar~\shortcite{CPL14}                             & \eventConstantDepth(1)        & \selectorForward \\
  \lrastar~\shortcite{Mandalika18}         				& \eventConstantDepth($\alpha$)   & \selectorForward  \\ \bottomrule
\end{tabulary}
\caption{Equivalence of \glrastar to existing lazy algorithms}
\label{tab:equivalence}
\end{table}


\subsection{Analysis}
\label{sec:framework_analysis}

For any choice of \Event and \Selector, \glrastar is complete and correct. 

\begin{thm}[\textbf{Completeness}]
\label{theorem:complete}
Let \Event be a function that on halting ensures there is at least one unevaluated edge on the current shortest
path or that the goal is reached.
Let \Selector be a function that evaluates at least one
unevaluated edge (if it exists).
\glrastar implemented using \textsc{ExtendTree}(\Event) and \textsc{EvaluateEdges}(\Selector) on a finite graph $\graph$ is complete.
\end{thm}


\begin{proof}{}
In each iteration, \textsc{ExtendTree}(\Event) ensures there is atleast one unevaluated edge on the shortest path (unless the goal has been reached). The \textsc{EvaluateEdges}(\Selector) evaluates atleast one edge. Since there are a finite number of edges, the algorithm will eventually terminate. 
\end{proof}

\begin{thm}[\textbf{Correctness}]
\label{theorem:correct}
If the heuristic $h(\vertex, \vg)$ is admissible, then \glrastar terminates with the shortest feasible path.
\end{thm}


\begin{proof}{}
Let $\Path^*$ be the shortest feasible path with respect to $w(\cdot)$ and world $\world$. 
For any vertex $\vertex^* \in \Path^*$, we denote its f-value to be $f(\vertex^*) = w(\Path_{\vs, \vertex^*}) + h(\vertex^*, \vg)$, where $\Path_{\vs, \vertex^*}$ is the subpath from the start to vertex $\vertex^*$.
As  our heuristic function is admissible, we have that $f(\vertex^*) \leq w(\Path^*)$.
Recall that in each iteration, the inner \textsc{GetShortestPathToLeaf~()} returns a vertex $\vertex_{\rm ret}$ with the smallest f-value among all the leaves of the tree $\tau_{\rm lazy}$.
Let $\vertex^*_{\rm leaf}$ be the leaf vertex on $\tau_{\rm lazy}$ that lies on the shortest feasible path $\Path^*$.
Hence, $f(\vertex_{\rm ret}) \leq f(\vertex^*_{\rm leaf}) \leq  w(\Path^*)$. 
If \glrastar terminates with $\vertex_{\rm ret}$, this implies $\Path_{\vs, \vertex_{\rm ret}}$ is verified to be feasible and $\vertex_{\rm ret} = \vg$. Let the verified path that is returned be $\Path_{\rm ret}$ such that $\Path_{\rm ret} = \Path_{\vs, \vg}$. In that case $w(\Path_{\rm ret}) = f(\vg) \leq w(\Path^*)$.
\end{proof}

\lazySP with the \selectorForward selector was proved to be \emph{edge optimal\footnote{See \cite{Mandalika18} for the formal computational model}} in the class of all shortest path algorithms that use a \selectorForward selector~\cite{Mandalika18}. We now show that \glrastar lets us derive another algorithm that is \emph{also edge-optimal} but reduces number of vertex rewires.

\begin{thm}[\textbf{Edge Optimality}]
\label{theorem:optimal}
\glrastar evaluates the same number of edges $\Eeval$ as \lazySP, i.e., is edge optimal, while having a smaller number of vertex rewires $\Vexp$ under the following setting:
\begin{enumerate}
	\item Heuristic: Distance on the unevaluated graph $h_{\graph}(\vertex,\vg)$
	\item \Event: \eventHeuristicProgress
	\item \Selector: \selectorForward
\end{enumerate}
\end{thm}


\begin{proof}{}
We are going to prove this via induction over iterations of \lazySP and \glrastar. In each iteration cycles through algorithm by invoking \textsc{EvaluateEdges~(\Selector)}, \textsc{ExtendTree~(\Event)} and \textsc{SelectShortestSubpath~()}. 

At iteration $i$, let $\EevalLSP^i$ and $\VexpLSP^i$ be the edges evaluated and vertex rewired respectively by \lazySP. Let $\Path^i$ be the candidate shortest path. 

Let $\EevalGLS^i$ and $\VexpGLS^i$ be the edges evaluated and vertices rewired, respectively, by \glrastar at iteration $i$. Let $h_{\graph} (\vertex, \vg)$ be the heuristic used by the search which corresponds to the distance on the graph $\graph$. Let $\vertex^i$ be the leaf vertex corresponding to the current shortest subpath from the start $\Path_{\vs, \vertex^i}$. This implies $\vertex^i$ corresponds to the vertex with the smallest f-value $w(\Path_{\vs, \vertex^i}) + h_{\graph} (\vertex^i, \vg)$. 

We also introduce the lazy edge status function $\world_{\rm lazy}(\Path, \Eeval)$ which determines if a path $\Path$ is valid depending on edges evaluated thus far in $\Eeval$.

Following are the conditions for the induction:
\begin{enumerate}
	\item[A] Both algorithms have the same set of evaluated edges $\EevalGLS^i = \EevalLSP^i$.
	\item[B] Both algorithms share the same subpath $\Path_{\vs, \vertex}^i \subseteq \Path^i$.
\end{enumerate}

For $i=1$, $\EevalGLS^1 = \EevalLSP^1$ because no edges have been evaluated. Hence (A) is true. Since $h_{\graph} (\vertex, \vg)$ is the distance on the unevaluated graph, the leaf vertex $\vertex^i$ considered by \glrastar lies on $\Path^1$, i.e. $\Path_{\vs, \vertex}^1 \subseteq \Path^1$. Hence (B) is true.

We will show these conditions hold for $i+1$.

Since both \lazySP and \glrastar use \selectorForward, share the same subpath (A) and have the same evaluation status (B) - they both evaluate the same edge $\edge$. Both algorithms increase their evaluated set $\EevalLSP^{i+1} = \EevalGLS^{i+1} \gets \EevalGLS^{i} \cup \edge$. Hence (A) holds. 

If $\edge$ is valid, neither algorithms rewire vertices. However, if an edge is in collision, \lazySP rewires at least the remainder of the path $\Path^{i+1}$. \glrastar does not have to rewire the remainder of the subpath $\Path_{\vertex^i, \vg}$ as it was never expanded during the search. Hence \glrastar can only result in smaller rewires, i.e. $|\VexpGLS^{i+1}| \leq |\VexpLSP^{i+1}| - |\Path_{\vertex^i, \vg}|$.

We will now show that $\Path_{\vs, \vertex^{i+1}} \subseteq \Path^{i+1}$. 

\lazySP finds the next candidate shortest path $\Path^{i+1}$ by solving the following search problem
\begin{equation}
\label{eq:lsp_opt}
\begin{aligned}
\Path^{i+1} \gets \argmin_{\Path} \;\;  & w(\Path) \\
\text{s.t.} \;\; & \world_{\rm lazy}(\Path, \EevalLSP^{i+1})=1 \\
\end{aligned}
\end{equation}

\glrastar invokes the \textsc{ExtendTree~(\Event)} which proceeds till \eventHeuristicProgress toggles off the search. The search stops at vertex $\vertex^{i+1}$ which satisfies the following:
\begin{equation}
\label{eq:gls_opt}
\begin{aligned}
\vertex^{i+1} \gets \argmin_{\vertex} \;\;  & w(\Path_{\vs,\vertex}) + h_{\graph} (\vertex, \vg) \\
\text{s.t.} \; \;& h_{\graph} (\vertex, \vg) < \hmin \\
\end{aligned}
\end{equation}
Note that $\world_{\rm lazy}(\Path_{\vs,\vertex}, \EevalGLS^{i+1}) = 1$, i.e. the subpath from the start to any vertex is valid according to the lazy estimate.

By definition, the heuristic $h_{\graph} (\vertex, \vg) = w(\Path_{\vertex,\vg})$ is the weight of the shortest path on the unevaluated graph $\Path_{\vertex,\vg}$. The heuristic progress threshold $\hmin$ is by definition the minimum heuristic value of the child vertex of any evaluated edge, i.e. $\hmin = \min_{ (u',v') \in \EevalGLS^{i+1} } h_{\graph} (v', \vg)$. Since $h_{\graph} (\vertex, \vg)$ is consistent, $ h_{\graph} (\vertex, \vg) < \hmin < \min_{ (u',v') \in \EevalGLS^{i+1} } h_{\graph} (v', \vg)$ implies that none of the edges $(u',v') \in \Path_{\vertex,\vg}$ belong to $\EevalGLS^{i+1}$ have been evaluated. This means that the subpath to goal is valid according to the lazy estimate, i.e. $\world_{\rm lazy}(\Path_{\vertex,\vg}, \EevalGLS^{i+1}) = 1$. 

Hence (\ref{eq:gls_opt}) can be re-written as 
\begin{equation}
\label{eq:gls_opt_new}
\begin{aligned}
\vertex^{i+1} \gets \argmin_{\vertex} \;\;  & w(\Path_{\vs,\vertex}) + w(\Path_{\vertex,\vg}) \\
\text{s.t.} \;\; & \world_{\rm lazy}(\Path_{\vs,\vertex}, \EevalGLS^{i+1}) = 1 \\
&\world_{\rm lazy}(\Path_{\vertex,\vg}, \EevalGLS^{i+1}) = 1 \\
\end{aligned}
\end{equation}

Since $\EevalGLS^{i+1} = \EevalLSP^{i+1}$, (\ref{eq:lsp_opt}) and (\ref{eq:gls_opt_new}) are the same optimization. Hence $\Path_{\vs, \vertex^{i+1}} \subseteq \Path^{i+1}$ and (B) holds. As a result, the induction holds.

This process continues till both algorithms discover the shortest feasible path $\Path^*$ at the end of iteration $N$. Both evaluate the same number of edges $\EevalLSP^{N+1} = \EevalGLS^{N+1}$. But \glrastar saves on more vertices being rewired than \lazySP, i.e. $|\VexpGLS^N| \leq |\VexpLSP^N| - \sum_{i=1}^N |\Path_{\vertex^i, \vg}|$. 
\end{proof}

\begin{cor}
There is a graph $\graph$ for which the number of vertex rewires $\Vexp$ for \lazySP over \glrastar is linear over logarithmic. 
\end{cor}

\begin{proof}{}

\begin{figure}[!htbp]
\centering
\includegraphics[width=\columnwidth]{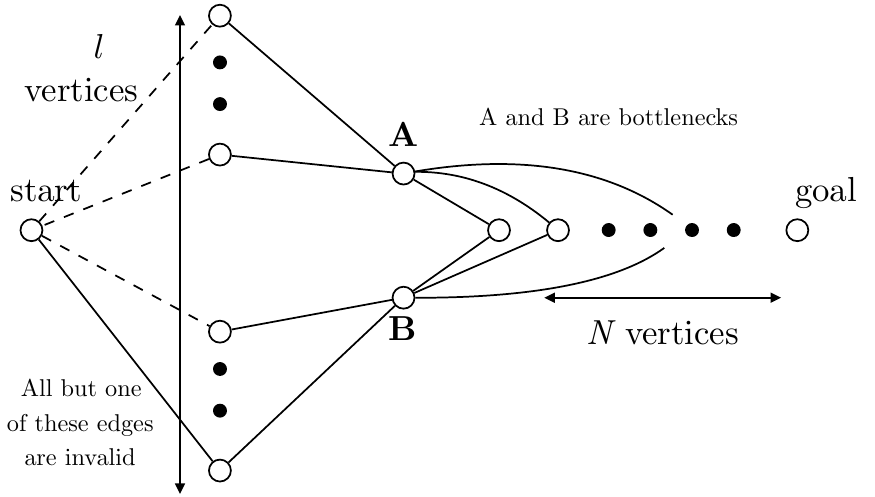}
\caption{ Counter examples}
\label{fig:counter}
\end{figure}

We are going to construct a counter example that shows a particularly bad case of vertex rewiring undertaken by \lazySP.

\paragraph{Scenario.}

Consider the graph in Fig.~\ref{fig:counter}. It has a set of $l$ vertices connected to the start. The upper half of the $l$ vertices are connected to vertex $A$. The lower half is connected to $B$. Each of $A$ and $B$ is connected to a chain of $N$ vertices going to the goal. $\abs{\vertexSet} = N+l$, $\abs{\edgeSet} = 3N + 2l - 1$.

The graph is such that one of $l$ edges connected to the start is valid, the rest is invalid. The remaining edges are all valid.

The weights of the graph are such that the shortest path alternates between the top and bottom halves of the graph. Assume that all $l-1$ shortest paths are invalid and the last one is valid. Finally assume that $A$ and $B$ alternate being the optimal parent to the $N$ vertices.

\paragraph{\lazySP computation}.

For \lazySP, the only computation is vertex rewiring. The graph is such that successive shortest paths alternate between the upper and lower halves. The shortest paths in the upper half pass through A and lower half pass through B. Hence every edge that is invalidated, causes all $N$ vertices to rewire to either vertex $A$ or $B$. This is the optimistic thrashing scenario explained in LazyPRM*. Since $l$ edges have to be invalidated, the number of rewires is $O(Nl)$
 
\paragraph{ \glrastar computation}

There are two computation steps to account for - heuristic computation and vertex rewiring. 

The heuristic computation is a Djikstra operation. 
\begin{equation}
\begin{aligned}
& O( \abs{\vertexSet} \log \abs{\vertexSet} + \abs{\edgeSet}) \\
& O( (N+l) \log (N+l) + 3N + 2l - 1) \\
& O( (N+l) \log (N+l)) \\
\end{aligned}
\end{equation}
Since the search never proceeds beyond the first set of edges, the amount of vertex rewiring is $0$. 

Hence the complexity of \lazySP is $O(Nl)$ while $\glrastar$ is $O((N+l) \log (N+l))$. The ratio is linear over logarithmic growth. 

\end{proof}





%
\section{Leveraging Edge Priors in \glrastar}
\label{sec:approach}

The \glrastar framework is powerful because one can optimize \Event and \Selector to minimize computational costs while still retaining guarantees. Here, we show its expressive power in a scenario where we have additional side information, such as priors on the validity of edges. Such information can be collected from datasets of prior experience or generated from approximations of the world representation. 

\subsection{Modified Problem Formulation}
We assume that the validity of each edge is an independent Bernoulli random variable. We are given a vector of probabilities $\vectorp \in [0,1]^{\abs{\edgeSet}}$, such that $P(\world(\edge) = 1) = \vectorp(\edge)$, i.e., for each edge $\edge$, we have access to $\vectorp(\edge)$, which defines the probability of the edge being valid in the current world $\phi$.

We allow the shortest path algorithm $\ALG(\graph, \vs, \vg, \world, \vectorp)$ to leverage knowledge of edge probabilities $\vectorp$ to minimize the expected computation cost as follows:
\begin{equation}
\begin{aligned}
\label{eq: objective_function_definition}
\min 		 \;\;  & \expect{\world \sim \vectorp}{ C(\Eeval, \Vexp) } \\
\suchthat 	  \;\; & \Eeval,~\Vexp = \ALG(\graph, \vs, \vg, \world, \vectorp) \\
\end{aligned}
\end{equation}

\subsection{\Event and \Selector Design}

\paragraph*{Event.} The \Event restricts lazy search from proceeding beyond a point when the search is likely to be ineffective, i.e. to a point that potentially increases the amount of rewires \Vexp. One such case is when the current shortest subpath is likely to be in collision, i.e., the probability of being valid drops below a threshold $\delta$. We describe this event, \eventPathExistence (\texttt{SE}), in Algorithm~\ref{alg:event_definitions}. 
We show that we can bound the performance of this event.
\begin{thm}
\label{theorem:existence_priors}
For any \Selector, the expected planning time of \eventPathExistence($\delta$) can be upper bounded as:
\begin{equation}
K \left( \costEval \frac{1}{(1 - \delta)} + \costExp \frac{b \log (\delta)}{\log(\pmax)} \right)
\end{equation} 
where $K$ is the number of shortest-paths that are infeasible, $b$ is the maximum branching factor, and $\pmax$ is the maximum value of an edge prior. 
\end{thm}

\begin{proof}{}
We first describe the \glrastar algorithm with \eventPathExistence($\delta$). The algorithm searches till the probability of the current shortest subpath drops below $\delta$. It toggles edge evaluation which will either eliminate the subpath or check an edge such that the probability rises above $\delta$. The search continues forward. This repeats till the shortest path has been found. For this proof, we assume we have an oracular selector that can invalidate a subpath if it is truly invalid. 

We begin by upper bounding the number of edge evaluations. Let $\Path^*$ be the shortest feasible path. Let there be $K$ shorter paths than $\Path^*$ that are infeasible and that the algorithm has to eliminate. Since we are showing an upper bound, we can relax the condition that the paths have overlapping edges since they will only reduce edge evaluations (eliminating one implies the other is eliminated).

Consider one of $K$ paths that we have to eliminate. If we pick an edge from the subpath, with probability $1-\delta$ we will find a witness that the path is invalid. A selector either invalidates a subpath with probability $1 - \delta$ or results in a wasted edge evaluation. This process is repeated till a path is eventually eliminated. The expected number of edge evaluated to eliminate the path is:
\begin{equation}
\begin{aligned}
\expect{\vectorp}{\Eeval} &\leq(1-\delta) + 2 \delta (1-\delta) + 3 \delta^2 (1-\delta) + \ldots \\
&\leq(1 - \delta) \left( 1 + 2 \delta + 3 \delta^2 + \ldots \right) \\
&\leq(1 - \delta) \frac{1}{(1 - \delta)^2} \\
&\leq\frac{1}{(1 - \delta)}
\end{aligned}
\end{equation}
Hence the total expected cost of edge evaluation is bounded by $\costEval K \frac{1}{(1 - \delta)}$. Note as $\delta \rightarrow 1$, this term goes to $\infty$. This is backed by the intuition that triggering the event often results in increased edge evaluation.

We will now upper bound the number of vertex rewiring. We assume that the search is using as heuristic the distance on the graph $h_{\graph}(\vertex, \vg)$. 
Hence when one of the $K$ subpaths are eliminated, only the vertices of that subpath is rewired. Since we are deriving an upper bound, we will ignore overlap between subpaths (which can only help). 

Consider one of $K$ paths that we have to eliminate. Let $\pmax$ be the maximum probability of an edge being valid. Then the maximum length of any subpath $L(\delta)$ is 
\begin{equation}
\begin{aligned}
\pmax^{L(\delta)} \geq \delta \\ 
L(\delta) & \leq \frac{\log (\delta)}{\log(\pmax)}
\end{aligned}
\end{equation}

When the subpath is eliminated, the rewiring is restricted only to vertices belonging to the subpath. Hence the maximum vertex rewire that can occur is $\Vexp = b L(\delta)$ where $b$ is the maximum branching factor. A selector either invalidates a subpath with probability $1 - \delta$ and results in rewiring or the process continues without any penalty. The expected number of vertices rewired before the path is eliminated can be upper bounded:
\begin{equation}
\begin{aligned}
\expect{\vectorp}{\Vexp} &\leq(1-\delta)b L(\delta) + \delta (1-\delta) b L(\delta) +  \ldots \\
&\leq(1 - \delta) b L(\delta) \left( 1 + \delta + \delta^2 + \ldots \right) \\
&\leq(1 - \delta) b L(\delta) \frac{1}{(1 - \delta)} \\
&\leq b L(\delta) \\
&\leq \frac{b \log (\delta)}{\log(\pmax)}
\end{aligned}
\end{equation}
Hence the total expected cost of vertex rewiring is upper bounded by $\costExp K b \frac{\log (\delta)}{\log(\pmax)}$. Note as $\delta \rightarrow 0$, this term goes to $\infty$. This is backed by the intuition that triggering the event less often results in increased vertex rewiring.  
\end{proof}


Low values of $\delta$ result in lower edge evaluations but more edge rewiring, and vice-versa.
\begin{cor}
\label{theorem:existence_delta}
There exists a critical threshold $\delta \in (0,1)$ that upper bounds the expected computational cost.
\end{cor}


\begin{proof}{}
The total expected planning time can be bounded as:
\begin{equation}
\begin{aligned}
\expect{\vectorp}{\costTotal(\Eeval, \Vexp)} &= \costEval \abs{\Eeval} + \costExp \abs{\Vexp} \\
&= \costEval K \frac{1}{(1 - \delta)} + \costExp K \frac{b \log (\delta)}{\log(\pmax)} \\
&= K \left( \costEval \frac{1}{(1 - \delta)} + \costExp \frac{b \log (\delta)}{\log(\pmax)} \right)
\end{aligned}
\end{equation}

We will now show that there exists a critical point $\delta$ that minimizes this. Solving for that critical point, we have:
\begin{equation}
\begin{aligned}
\frac{\partial}{\partial \delta} \left( K \left( \costEval \frac{1}{(1 - \delta)} + \costExp \frac{b \log (\delta)}{\log(\pmax)} \right) \right) &= 0 \\
\frac{\costEval}{ (1-\delta)^2 } +  \frac{\costExp}{\log(\pmax)} \frac{b}{\delta} &= 0 \\
(1-\delta)^2 - \frac{\costEval}{b \costExp} \log \left(\frac{1}{\pmax}\right) \delta &= 0 \\
\delta^2 - \left(\frac{\costEval}{b \costExp} \log (\frac{1}{\pmax}) + 2\right) \delta + 1 &= 0\\
\end{aligned}
\end{equation}

Let $\eta = \left(\frac{\costEval}{b\costExp} \log (\frac{1}{\pmax}) + 2\right)$. The critical point is:
\begin{equation}
\delta = \frac{\eta - \sqrt{\eta^2 - 4}}{2} 
\end{equation}
When $\costEval \approx \costExp$, $\eta$ is close to $2$ and $\delta \rightarrow 1$. When $\costEval \gg \costExp$, we have $\eta \gg 2$. Hence, the critical point is:
\begin{equation}
\begin{aligned}
\delta &= \frac{\eta - \sqrt{\eta^2 - 4}}{2} \\
&= \frac{\eta \left(1 - \sqrt{1 - \frac{4}{\eta^2}} \right)}{2} \\
&= \frac{\eta \left(1 - \left(1 - \frac{4}{2\eta^2} \right) \right)}{2} \\
&= \frac{1}{\eta} \\
&= \frac{1}{  \left(\frac{\costEval}{b\costExp} \log (\frac{1}{\pmax}) + 2\right) }\
\end{aligned}
\end{equation}

The critical point is inversely proportional to the ratio $\frac{\costEval}{\costExp}$. 

\end{proof}

\paragraph*{Selector.} The \Selector invalidates as many subpaths as quickly as possible, which restricts the size of \Eeval. 
One strategy for doing so is to invalidate the current subpath as quickly as possible. 
We describe a selector, \selectorGreedy (\texttt{FF}), in Algorithm~\ref{alg:selector_definitions} that evaluates the edge on the subpath with the highest probability of being in collision. We show that this selector is the optimal strategy to invalidate a subpath.

\begin{thm}
Given a path \Path{}, \selectorGreedy{} minimizes the expected number of edges from \Path that must be evaluated to invalidate \Path.
\label{theorem:greedy_selector}
\end{thm}


\begin{proof}{}
Given a path $\Path$, and a sequence of edges $S = \{e_1, e_2, \ldots, e_n\}$ belonging to the path,
and the corresponding priors of the edges being valid $(p_1, p_2, \ldots, p_n)$, 
let the expected number of edge evaluations to invalidate the \Path be $\Eeval(S)$ which is given by
\begin{equation}
\begin{aligned}
\expect{\vectorp}{ \Eeval(S) } &= (1-p_1) + 2 p_1 (1-p_2) + \ldots \\
&= \sum_{l=1}^{n} \left( \prod_{m=1}^{l-1} p_m \right) \left(1-p_l\right)l 
\end{aligned}
\label{eqn:greedy_expected_edge_evaluations}
\end{equation}
Without loss of generality, let $p_i > p_{i+1}$ for a given $i$. Consider the alternate sequence of evaluations $S' = \{e_1, e_2, \ldots, e_{i+1}, e_i \ldots, e_n\}$ where the positions of the edges $e_i,~e_{i+1}$ are swapped. 
Consider the difference:
\begin{equation}
\begin{aligned}
& \expect{\vectorp}{ \Eeval(S) } - \expect{\vectorp}{ \Eeval(S') } \\
&= \ldots + \prod_{m=1}^{i-1} p_m \left[ (1-p_i) i + p_i(1-p_{i+1})(i+1) \right] + \ldots \\
&- \ldots + \prod_{m=1}^{i-1} p_m \left[ (1-p_{i+1}) i + p_{i+1}(1-p_{i})(i+1) \right] + \ldots \\
&= \prod_{m=1}^{i-1} p_m \left[-i(p_i - p_{i+1}) + (i+1)(p_i - p_{i+1}) \right] \\
&= \prod_{m=1}^{i-1} p_m (p_i - p_{i+1}) \\
& > 0
\end{aligned}
\label{eqn:greedy_expected_edge_evaluations}
\end{equation}
Since each such swap results in monotonic decrease in the objective, there exists an unique fixed point, i.e., the optimal sequence $S^*$ has $p_1 \leq p_2 \leq \ldots \leq p_n$.
\end{proof}

\subsection{Hypotheses}
\label{sec:hypothesis}

Based on our theoretical analysis and insight, we state three hypotheses that we intend to test:
\begin{hypothesis}
For any \Selector, the event \eventPathExistence{} requires less planning time compared to \eventShortestPath and \eventConstantDepth{}.
\label{hyp:1}
\end{hypothesis}
This follows from Theorem~\ref{theorem:existence_priors}, which upper bounds the planning time for \eventPathExistence. \eventShortestPath corresponds to $\delta = 0$ and can increase planning time. \eventConstantDepth{} has a fixed lookahead and does not adapt as priors change. 

\begin{hypothesis}
For any \Event, \selectorGreedy{} evaluates fewer edges than \selectorForward{} and \selectorAlternate{}.
\label{hyp:2}
\end{hypothesis}
This follows from Theorem~\ref{theorem:greedy_selector}, which shows that \selectorGreedy{} is optimal in expectation for eliminating a path. From \hypref{hyp:1} and \hypref{hyp:2}, we hypothesize that the combination of \eventPathExistence{} and \selectorGreedy{} will have the lowest planning time.

\begin{hypothesis}
The performance gain of \eventPathExistence{} over \eventShortestPath increases with both graph size and problem difficulty.
\label{hyp:3}
\end{hypothesis}
\eventShortestPath assumes that $\Vexp$ is negligible. As graph size increases, the size of vertices $\Vexp$ that \eventShortestPath rewires also increases. Similarly, as problem difficulty increases, so does the number of shortest paths that \eventShortestPath must invalidate, which also increases $\Vexp$. \eventPathExistence{}, on the other hand, makes no such assumption.

\section{Experiments}
\label{sec:experiments}


\paragraph{Algorithm Details.} 
We implemented 3 \Event{}s and 3 \Selector{}s described in Algorithms \ref{alg:event_definitions} and \ref{alg:selector_definitions} to get a total of $9$ algorithms. To analyze the trade-offs, we test on a diverse set of $\R^2$ datasets. We then finalize on 3 algorithms: \lazySP (\eventShortestPath,~\selectorGreedy), \lrastar(\eventConstantDepth,~\selectorGreedy) and \glrastar(\eventPathExistence,~\selectorGreedy). We evaluate these on a Piano Movers' problem in $SE(2)$ and manipulation problems in $\R^7$ using HERB \cite{Srinivasa2009}, a mobile robot with 7DoF arms. \footnote{Code is publicly available as an OMPL Planner at: \\ \url{https://github.com/personalrobotics/gls}}




\begin{figure*}[!t]
\centering
\includegraphics[width=\textwidth]{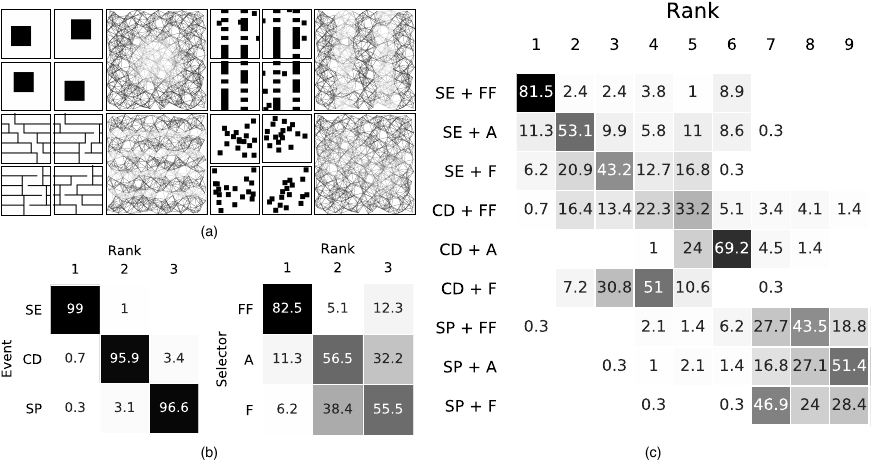}
\caption{(a) Samples and the prior for $\R^2$ datasets. (b) Events, Selectors and (c) Algorithms, ranked by planning times on $\R^2$ environments. Each cell indicates percentage of problems on which corresponding rank has been obtained.}
\label{fig:prior-ranking}
\end{figure*}

\begin{table*}[!htpb]
\small
\centering
\begin{tabulary}{\textwidth}{LCCCCCCCCC}
\toprule
	                             & {\glsspf}      & {\glsspa} & {\glsspff} & {\glscdf}      & {\glscda}  & {\glscdff} & {\glssef} & {\glssea} & {\glsseff}          \\ \midrule
\textbf{Square}                &                &           &            &                &            &            &           &           &                     \\
\emph{Total Planning Time}     &          0.331 & 0.454     & 0.372      & 0.221          & 0.259      & 0.222      &  0.171    &  0.161    &  \textbf{0.116}     \\
\emph{\# Edge Evaluations}  	 & 190 &  308 & 137 & 273 &  344.5 & 315.5 & 200.5 & 206 & 153 \\
\emph{\# Vertex Rewires}    	 &  7058.5 & 8502.5 & 9859 & 1076.5 & 641.5 & 69.5 & 1104.5 & 603.5 & 318.5 \\
\vspace{0.5mm}

\textbf{Two Wall}              &                &           &            &                &            &            &           &           &                   \\
\emph{Total Planning Time}     &          0.419 &	0.394 &	0.377 &	0.262 &	0.301 &	0.290 &	0.220 &	0.169 &	\textbf{0.161} \\
\emph{\# Edge Evaluations}  	 &  224 & 242.5 & 144 & 310 & 393 & 407 & 287 & 224.5 & 202.5 \\
\emph{\# Vertex Rewires}    	 & 9360 & 7997 & 9870 & 1594.6 & 914.5 & 165.5 & 697 & 435 & 711.5 \\
\vspace{0.5mm}

\textbf{Mazes}                 &                 &           &            &                &            &            &           &           &                     \\
\emph{Total Planning Time}     &          1.334 &	1.292 &	1.272 &	0.574 &	0.776 &	0.560 &	0.615 &	0.578 &	\textbf{0.337} \\
\emph{\# Edge Evaluations}  	 & 531 & 471 & 307.5 & 630 & 895 & 785.5 & 588 & 544.5 & 352.5 \\
\emph{\# Vertex Rewires}    	 &  34359 & 34379.5 & 37750 & 4769 & 5357.5 & 326 & 7266.5 & 7039.5 & 3213 \\
\vspace{0.5mm}

\textbf{Forest}                &                &           &            &                &            &            &           &           &                     \\
\emph{Total Planning Time}     &          0.277 &	0.267 &	0.269 &	0.574 &	0.776 &	0.559 &	\textbf{0.219} &	0.234 &	0.229 \\
\emph{\# Edge Evaluations}  	 & 174.5 & 184.5 & 165.5 & 306.5 & 342.5 & 450.5 & 190 & 220 & 174.5 \\
\emph{\# Vertex Rewires}    	 & 5524.5 & 4936 & 5467.5 & 579 & 346 & 95.5 & 3075 & 2855 & 3827.5 \\


\bottomrule
\end{tabulary}
\caption{Planning Time(millisec.) and number of operations by algorithms under \glrastar{}. (Median reported on 100 tests.)}
\label{tab:dataset_results}
\end{table*}

\begin{figure*}[!thb]
\centering
\includegraphics[width=0.9\textwidth]{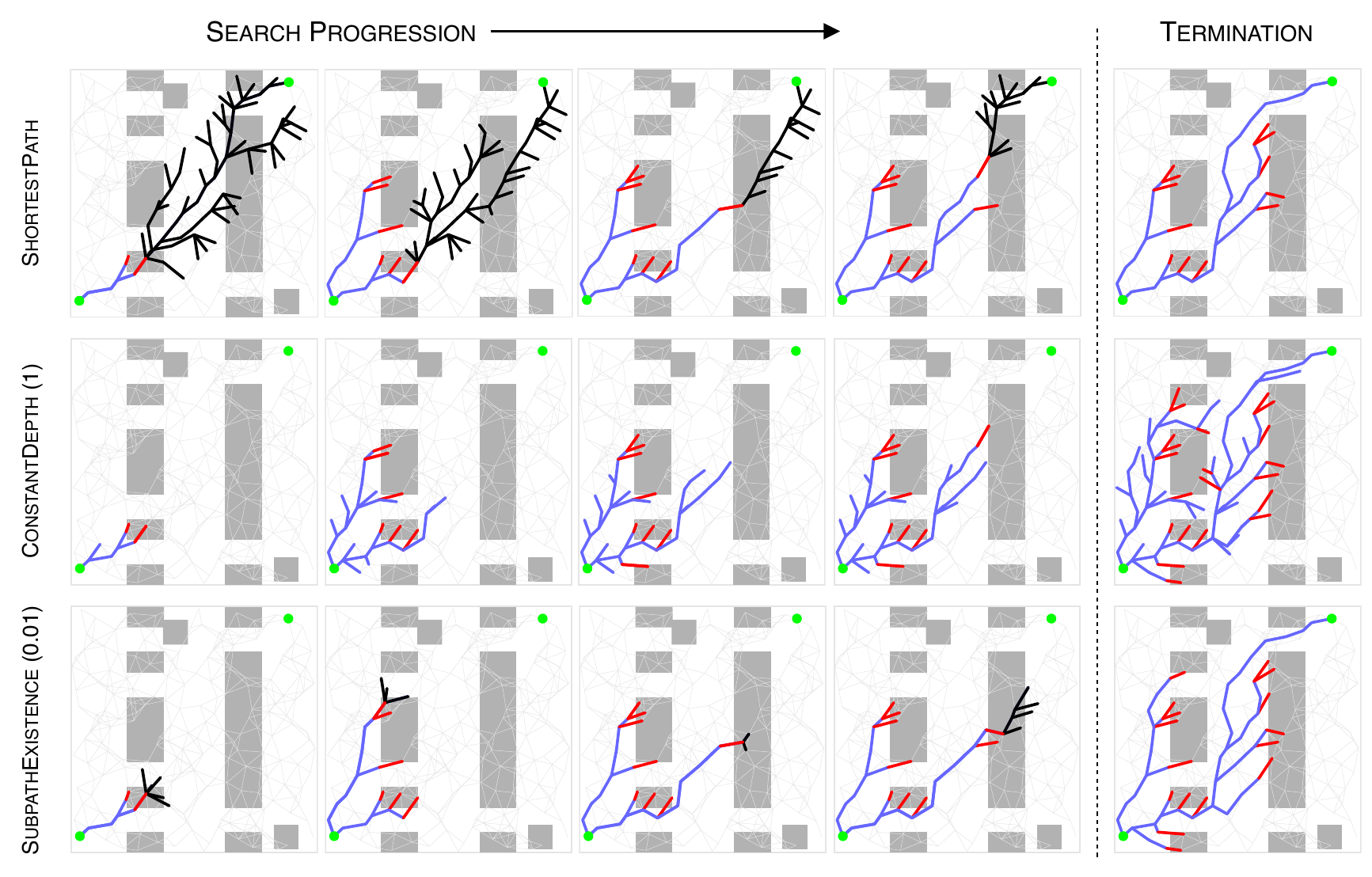}
\caption{Snapshots of search and evaluation by \glrastar{} with \selectorForward{} selector and different events. 
Edges evaluated to be valid (blue), invalid (red) and subtree of vertices to be rewired (black) are shown. From top to bottom at termination: the number of edge evaluations are (49, 97, 62) and the number of vertex rewires are (361, 21, 69).}
\label{fig:comparison}
\end{figure*}

\begin{figure}[!htb]
\centering
\begin{subfigure}[b]{0.45\linewidth}
  \centering
  \includegraphics[width=0.9\linewidth]{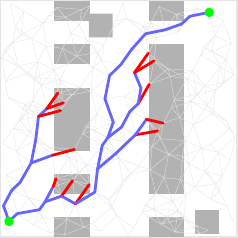}
\end{subfigure}
\hfill
\begin{subfigure}[b]{0.45\linewidth}
  \centering
  \includegraphics[width=0.9\linewidth]{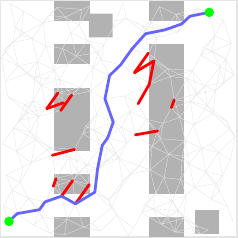}
\end{subfigure}
\caption{All valid (blue) and invalid (red) edges evaluated at termination of \glrastar{} with \selectorForward{} (49 edges) and \selectorGreedy{} (32 edges) with \eventPathExistence{}.}
\label{fig:forward_greedy_selectors}
\end{figure}

\begin{figure*}[!thb]
\centering
\begin{subfigure}[b]{0.256\textwidth}
  \centering
    \includegraphics[width=\textwidth]{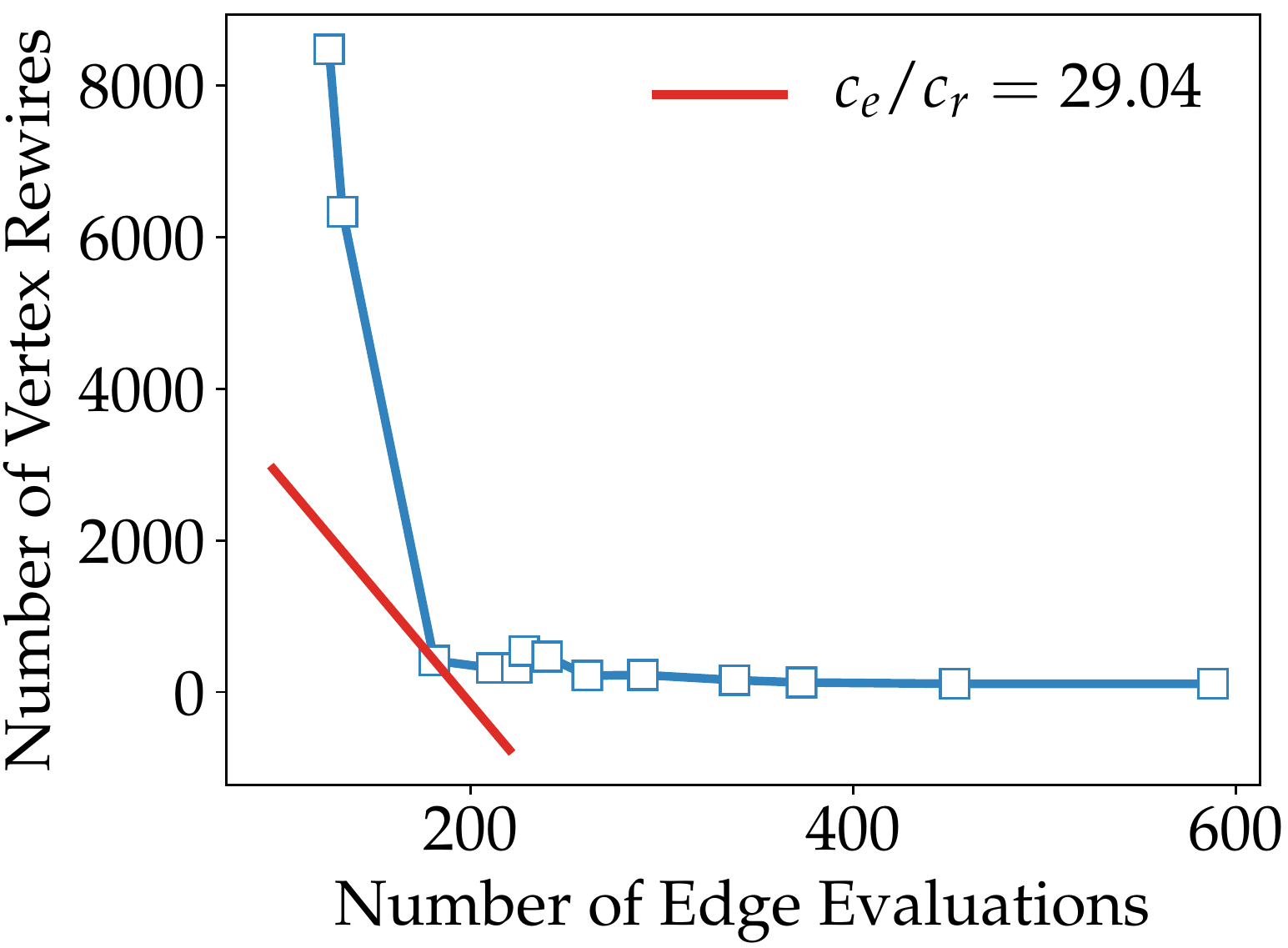} 
   \caption{\label{fig:performance:pareto}}
\end{subfigure}
\begin{subfigure}[b]{0.252\textwidth}
  \centering
  \includegraphics[width=\textwidth]{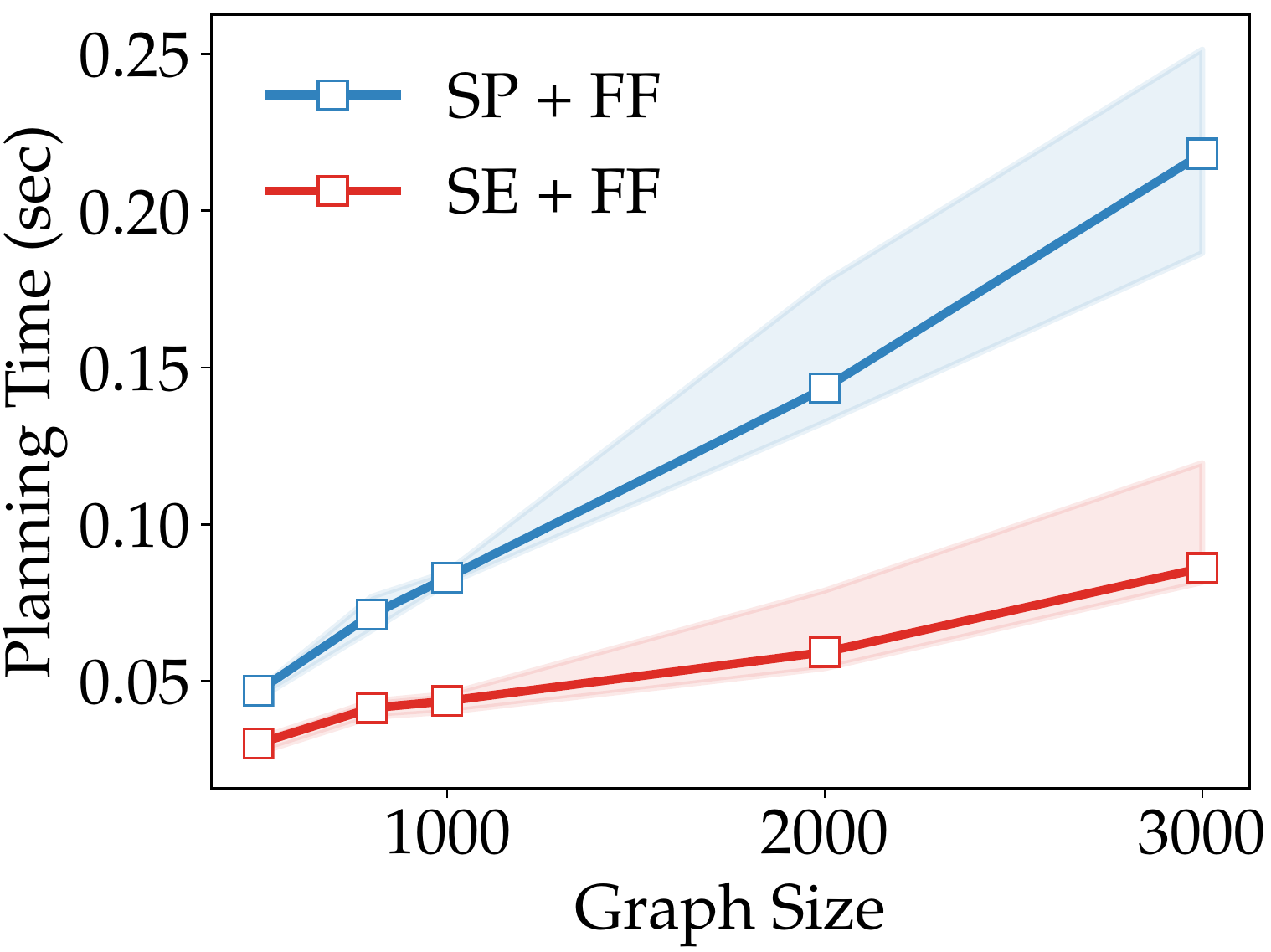}
     \caption{\label{fig:performance:graph_size}}
\end{subfigure}
\begin{subfigure}[b]{0.247\textwidth}
  \centering
  \includegraphics[width=\textwidth]{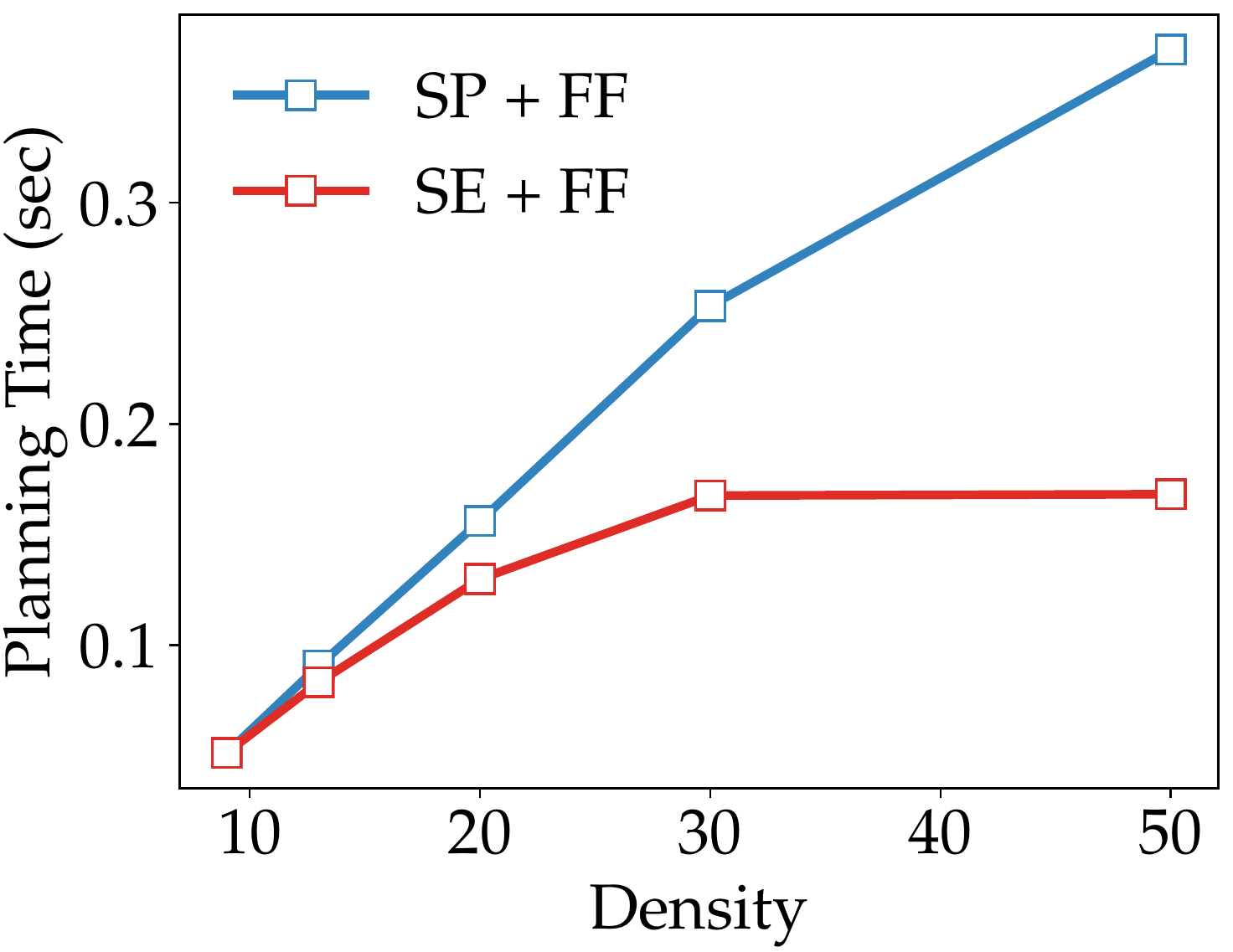}
     \caption{\label{fig:performance:difficulty}}
\end{subfigure}
\begin{subfigure}[b]{0.228\textwidth}
  \centering
  \includegraphics[height=0.9\textwidth]{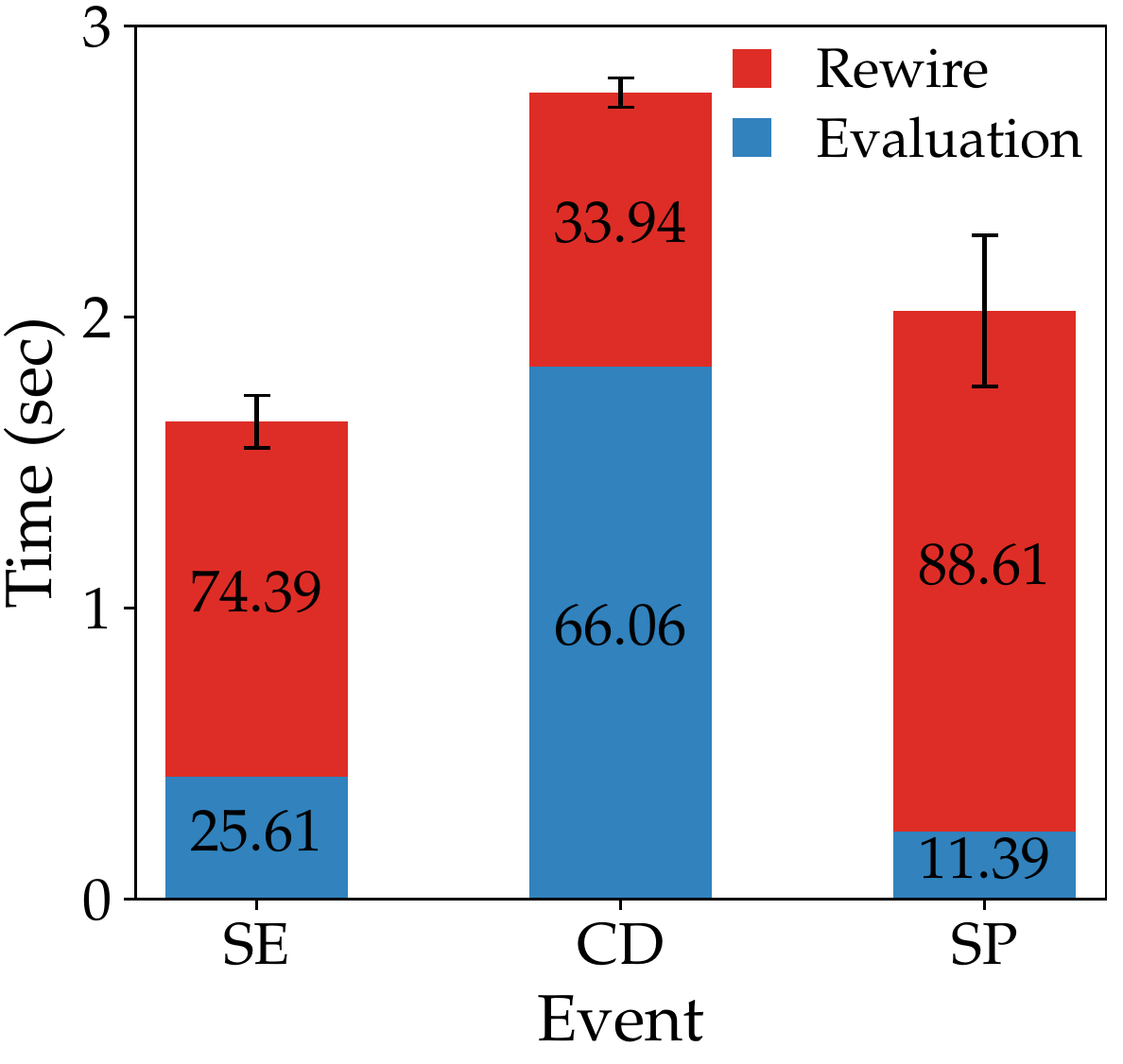}
     \caption{\label{fig:performance:bar_plot}}
\end{subfigure}
\caption{(a) Pareto curve obtained by varying $\delta$ in \eventPathExistence (b) Planning times as the size of the graph in TwoWall is increased. (c) Planning time as the density of obstacles in Forest is increased (d) Planning time in $\R^7$ problem (with 95\% C.I.)}
\label{fig:performance}
\end{figure*}

\paragraph{Analysis on $\R^2$ datasets.}

We use $5$ datasets of $\R^2$ problems from \cite{SC17}. Each dataset corresponds to different parametric distribution of obstacles from which we sample $1000$ worlds. A graph of 2000 vertices is sampled using a low dispersion sampler~\cite{Halton64} with an optimal connection radius~\cite{JSCP15}. Priors are computed by collision checking the graph on the training data and averaging edge outcomes. The prior and some samples from $\R^2$ datasets are shown in Fig.~\ref{fig:prior-ranking}(a). We pick one representative dataset, TwoWall, to show detailed plots.


We choose evaluation metrics (a) number of edge evaluations (b) number of vertex rewires and c) total planning time: weighted combination of (a), (b) (see Eq. \ref{eq: computation_cost_definition}). Since $\R^2$ problems are \emph{not expensive to evaluate}, we choose weights based on empirical data from manipulation planning problems in $\R^7$ (avg. eval time: $3.35\times 10^{-4}$s, avg. rewire time $1.1\times 10^{-5}$s, ratio $29.04$).

Finally, for parameter selection, we choose $\delta$ in \eventPathExistence{} from the pareto curve of vertices rewired vs edges evaluated computed on the training data. The slope of the line is the ratio of their relative cost -- the point of interesection corresponds to the $\delta:0.01$ that minimizes planning time. For \eventConstantDepth{}, we use the recommended value from \cite{Mandalika18}.

Table \ref{tab:dataset_results} shows the planning times of various algorithms under \glrastar{}. The planning times are the median quantities obtained from experiments over 100 different worlds sampled within the environment type. Fig.~\ref{fig:prior-ranking}(c) shows the ranking of the planning times of the algorithms across the 400 worlds considered across the four datasets (lower plannning time in a problem translates to a better rank). We note that \glrastar{} with \eventPathExistence{} and \selectorGreedy{} consistently outperforms remaining algorithms on a majority of environments.


We found strong evidence to support \hypref{hyp:1} - \eventPathExistence exhibits lowest planning times in 99\% of the problems (Fig.~\ref{fig:prior-ranking}(b, left)) Corresponding median planning times supporting the hypothesis are reported in Table \ref{tab:dataset_results}. Fig.~\ref{fig:comparison} shows a comparison of the events \eventShortestPath, \eventConstantDepth and \eventPathExistence (for the \selectorForward selector) on a problem from the TwoWall dataset. We can see that \eventShortestPath checks small number of edges but rewires significant portion of the search tree. The trend is reversed in \eventConstantDepth when using a depth of $1$. \eventPathExistence is able to balance both by exploiting priors - it triggers events when the search reaches the walls thus reducing rewires.

We also found strong evidence to support \hypref{hyp:2} - \selectorGreedy exhibits the lowest planning times in 83\% (Fig.~\ref{fig:prior-ranking}(b, right)) of the problems across the four datasets. In Table \ref{tab:dataset_results}, we note that for a given event, \selectorGreedy{} has the lowest planning time in majority of the datasets. Fig.~\ref{fig:forward_greedy_selectors} shows a comparison of \selectorForward and \selectorGreedy (for \eventShortestPath event) - \selectorGreedy quickly eliminates paths by checking the weakest link (supporting Theorem~\ref{theorem:greedy_selector}).

We found strong evidence to support \hypref{hyp:3}. Fig.~\ref{fig:performance:graph_size} shows that as graphs get larger, planning times of \eventShortestPath grows at a faster rate than \eventPathExistence. Fig.~\ref{fig:performance:difficulty} shows that as the density of obstacles increase, the planning times of \eventShortestPath grows linearly while \eventPathExistence eventually saturates. 


\begin{figure}
\centering
\includegraphics[width=0.6\linewidth]{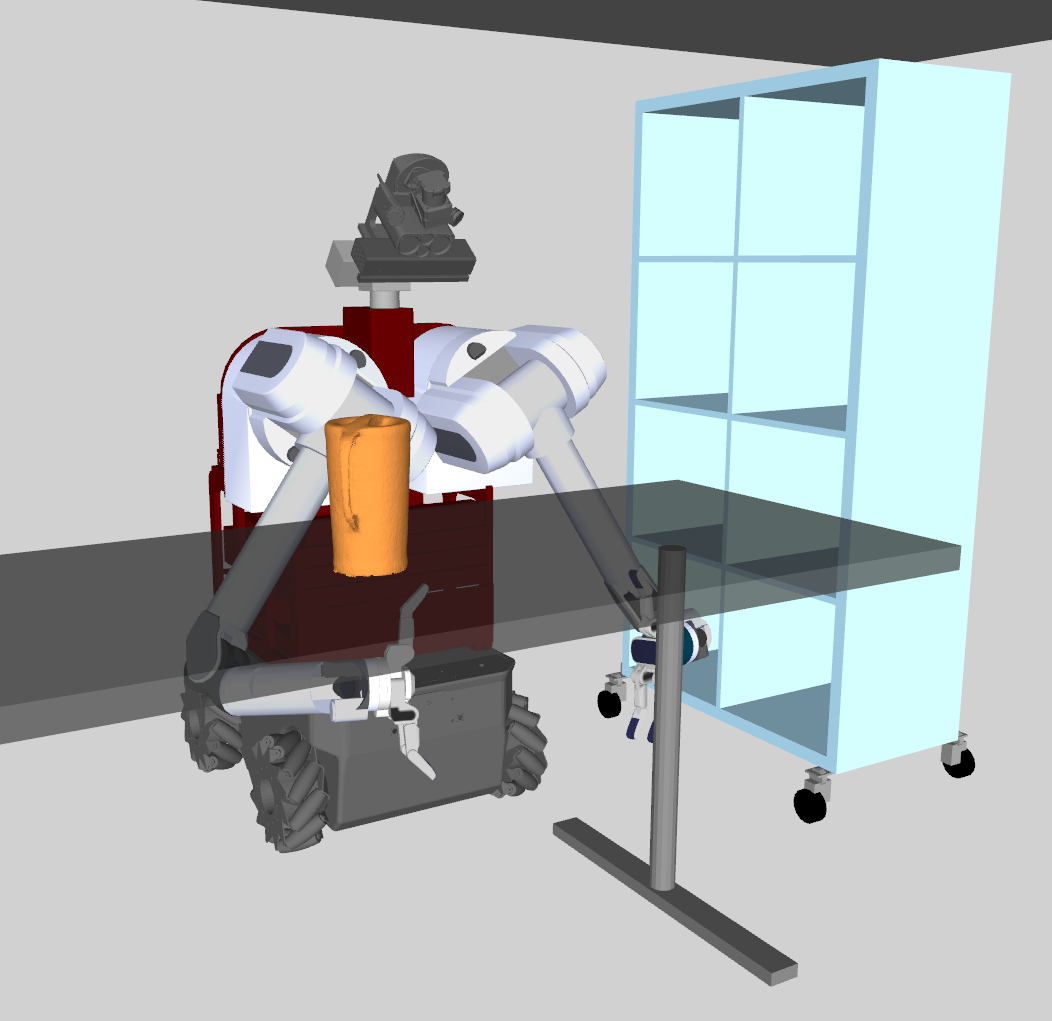}
\caption{HERB Task 1: Robot reaches into the shelf with its right arm while avoiding the table and the object.}
\label{fig:herb_scene_actual}
\end{figure}

\paragraph{Analysis on $SE(2)$ problems and $\R^7$ problems.}

\begin{table}[!htb]
\small
\centering
\caption{Mean Planning Times (in seconds) of \glrastar{}, \lazySP and \lrastar on $SE(2),~\R^7$ problems.}
\begin{tabulary}{\textwidth}{LCCCCCC}
\toprule
	                           & {\glrastar{}}  & {\lrastar} & {\lazySP} \\ \midrule
\textbf{Piano Movers'}       &                &           &          \\
\emph{Total Planning Time}   & \textbf{1.00}  &  1.13     & 1.25     \\
\emph{Edge Evaluation Time}  &         0.17   &  0.49     & 0.10     \\
\emph{Vertex Rewire Time}    &         0.83   &  0.64     & 1.15     \\
\vspace{0.5mm}
\textbf{HERB Task 1}         &                &           &          \\
\emph{Total Planning Time}   & \textbf{1.17}  &  1.81     & 1.53     \\
\emph{Edge Evaluation Time}  &         0.38   &  1.19     & 0.29     \\
\emph{Vertex Rewire Time}    &         0.79   &  0.62     & 1.24     \\
\vspace{0.5mm}
\textbf{HERB Task 2}         &                &           &          \\
\emph{Total Planning Time}   & \textbf{1.64}  &  2.77     & 2.02     \\
\emph{Edge Evaluation Time}  &         0.42   &  1.83     & 0.23     \\
\emph{Vertex Rewire Time}    &         1.22   &  0.94     & 1.79     \\

\bottomrule
\end{tabulary}
\label{tab:high_problems}
\end{table}

\begin{figure*}[!thb]
\centering
\begin{subfigure}[b]{0.32\textwidth}
  \centering
  \includegraphics[width=0.7\textwidth]{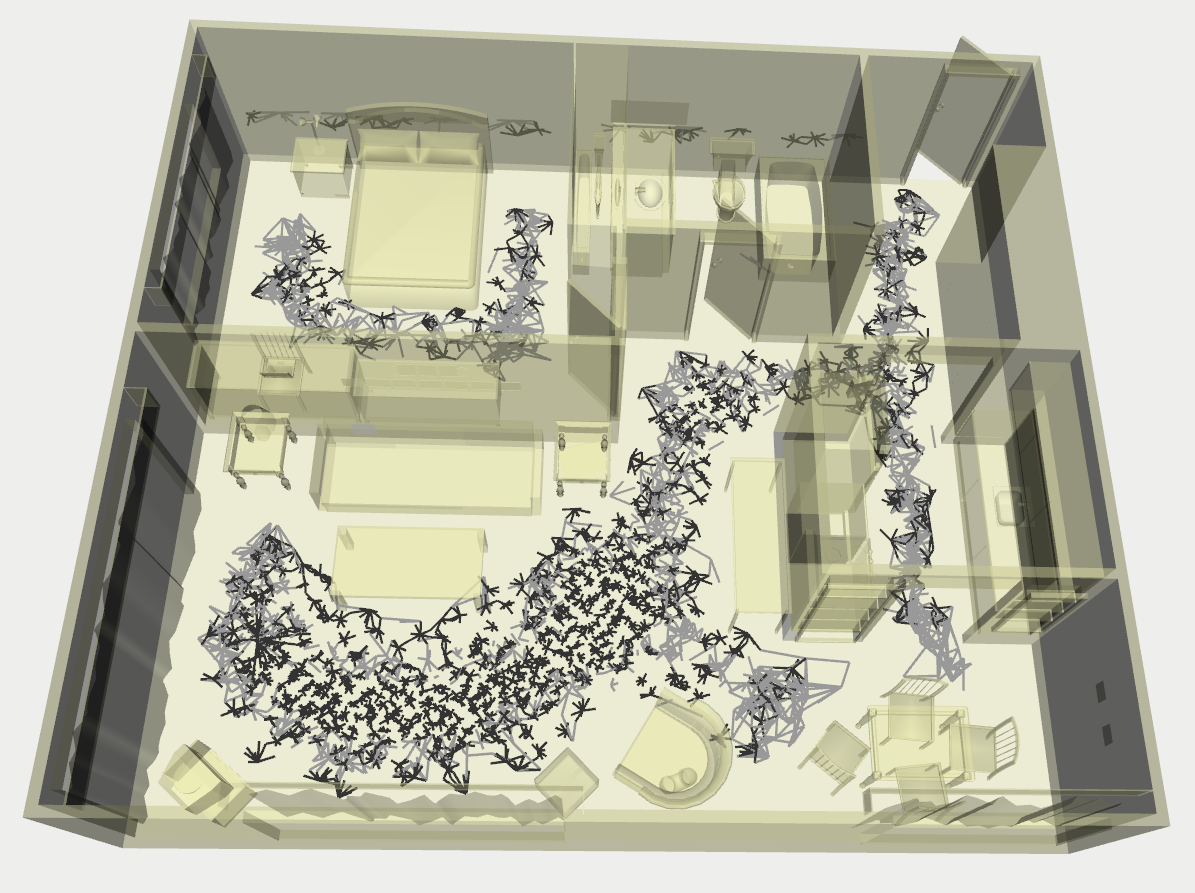}
  \subcaption{}  
\end{subfigure}
\begin{subfigure}[b]{0.32\textwidth}
  \centering
  \includegraphics[width=0.7\textwidth]{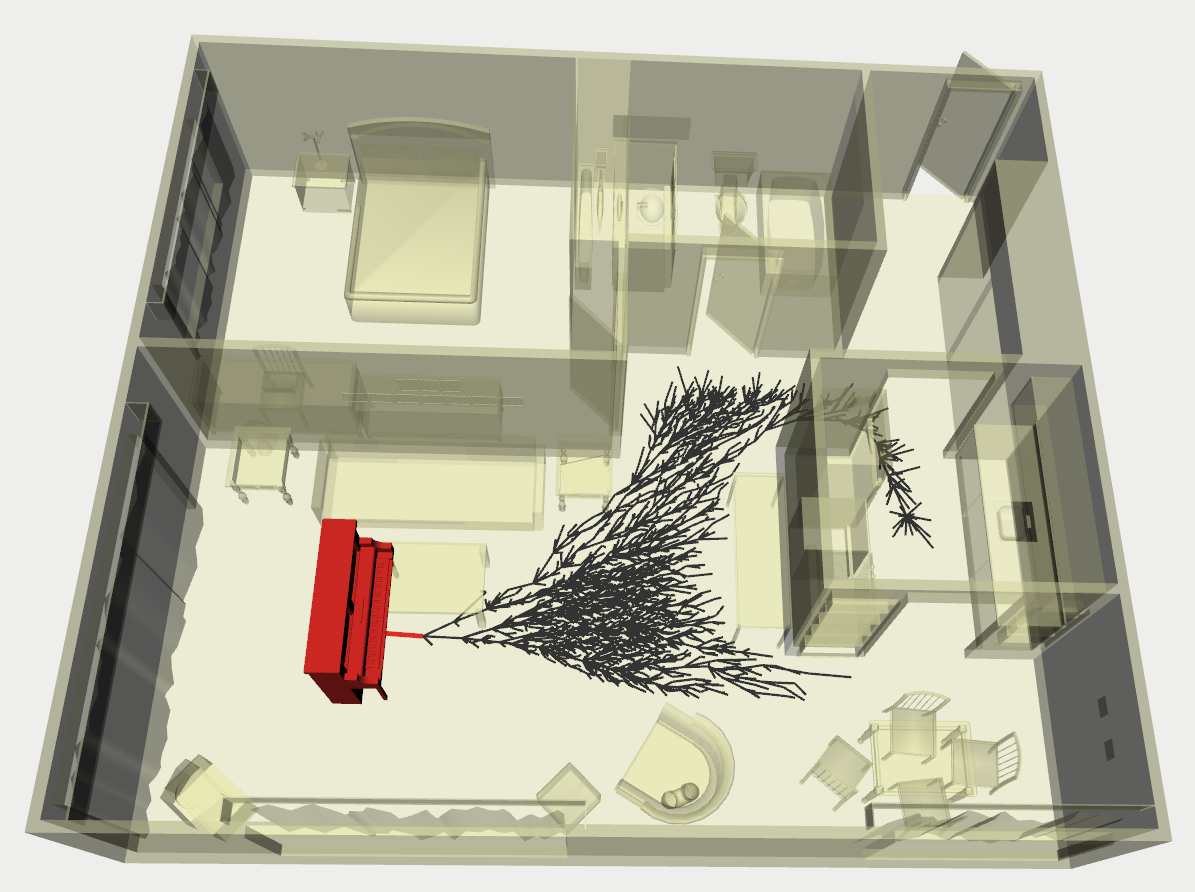}
  \subcaption{}
\end{subfigure}
\begin{subfigure}[b]{0.32\textwidth}
  \centering
  \includegraphics[width=0.7\textwidth]{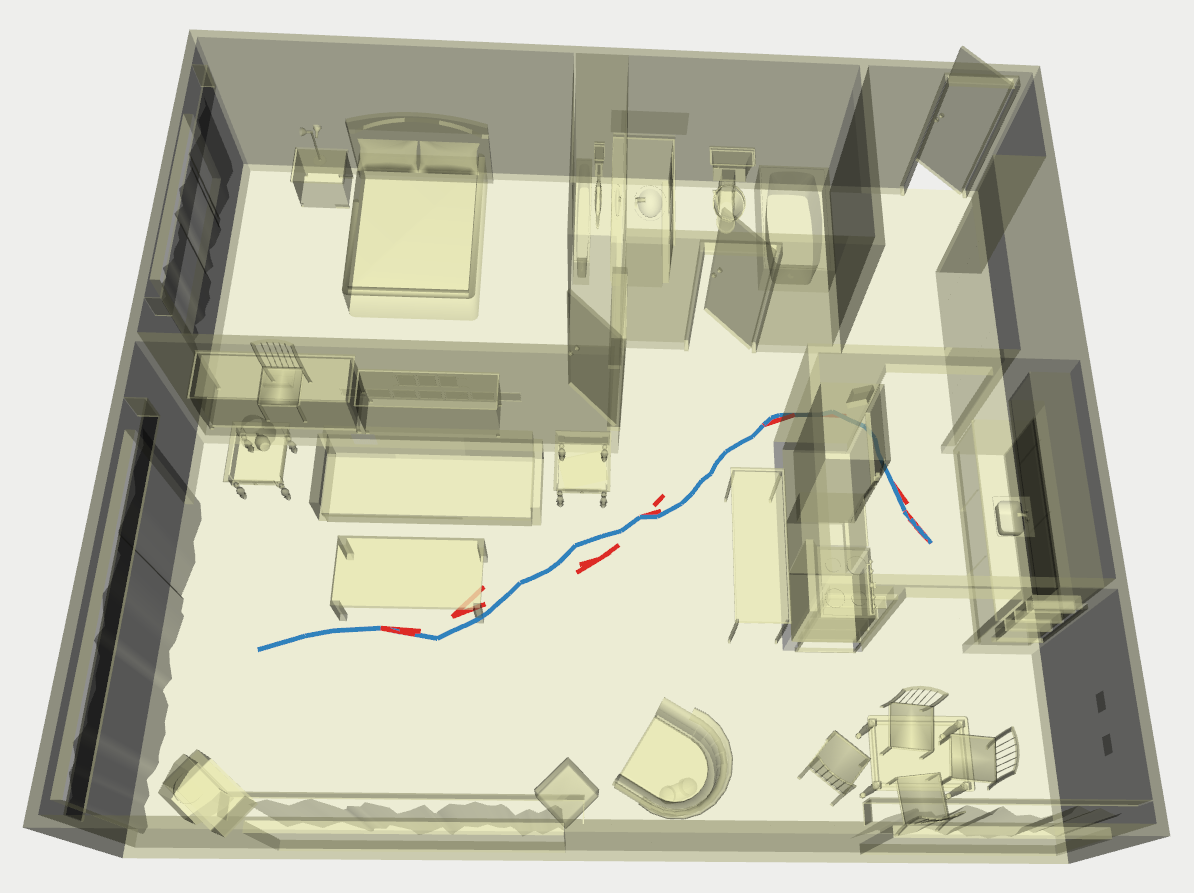}
  \subcaption{}
\end{subfigure}

\begin{subfigure}[b]{0.32\textwidth}
  \centering
  \includegraphics[width=0.7\textwidth]{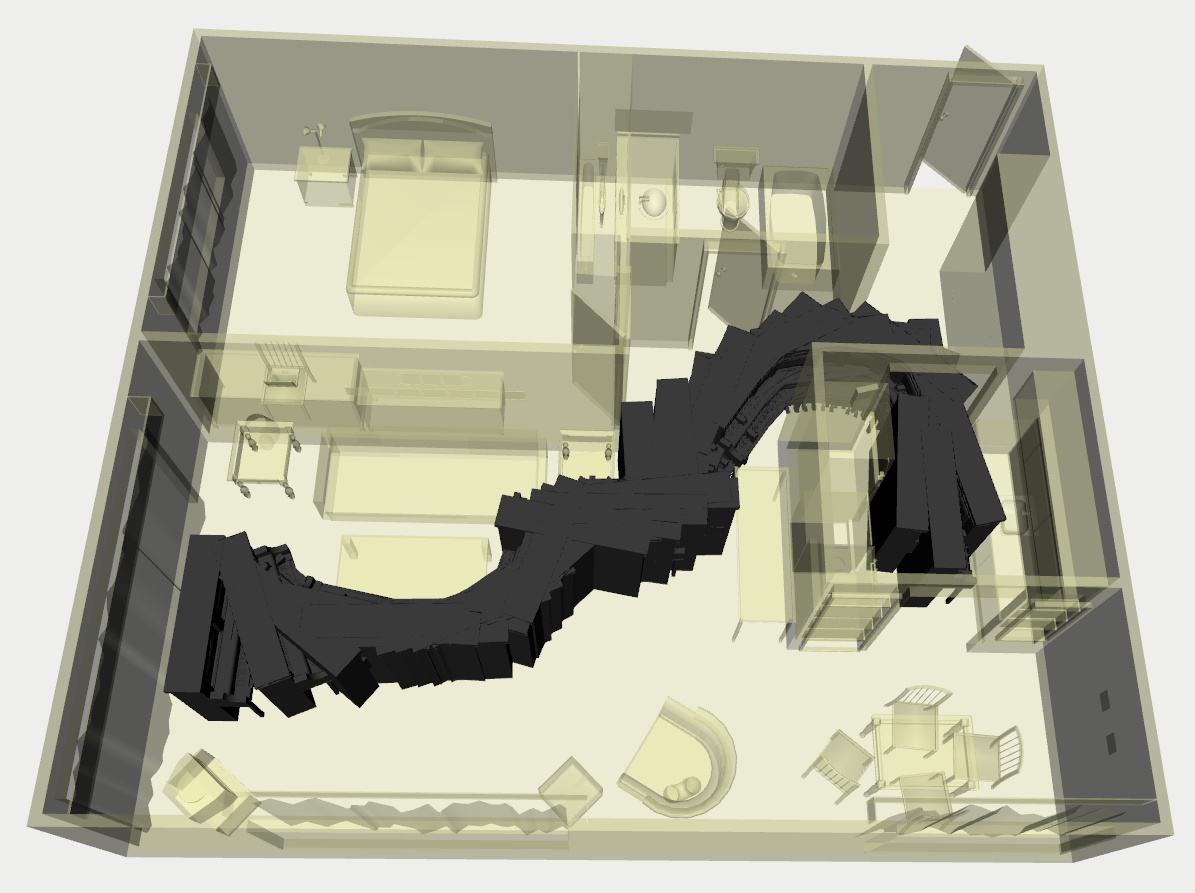}
  \subcaption{}
\end{subfigure}
\begin{subfigure}[b]{0.32\textwidth}
  \centering
  \includegraphics[width=0.7\textwidth]{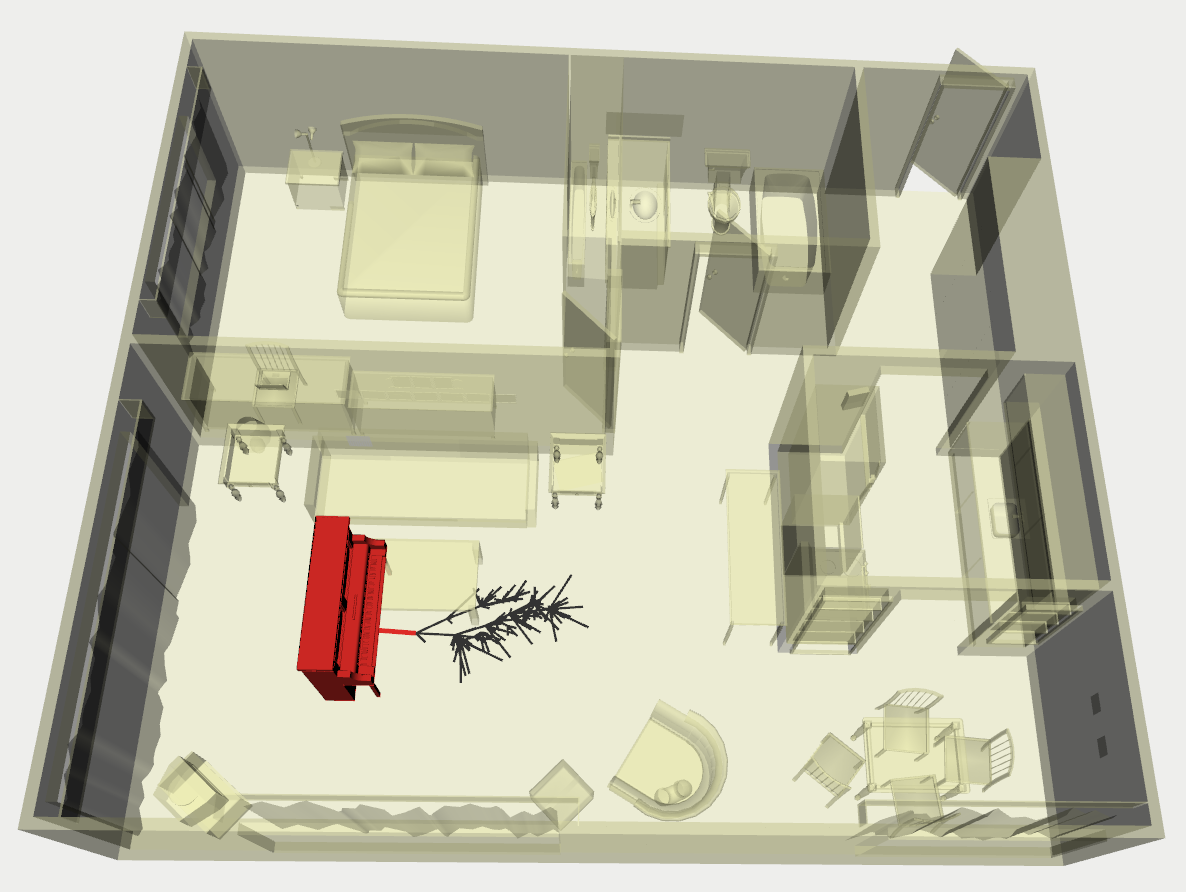}
  \subcaption{}
\end{subfigure}
\begin{subfigure}[b]{0.32\textwidth}
  \centering
  \includegraphics[width=0.7\textwidth]{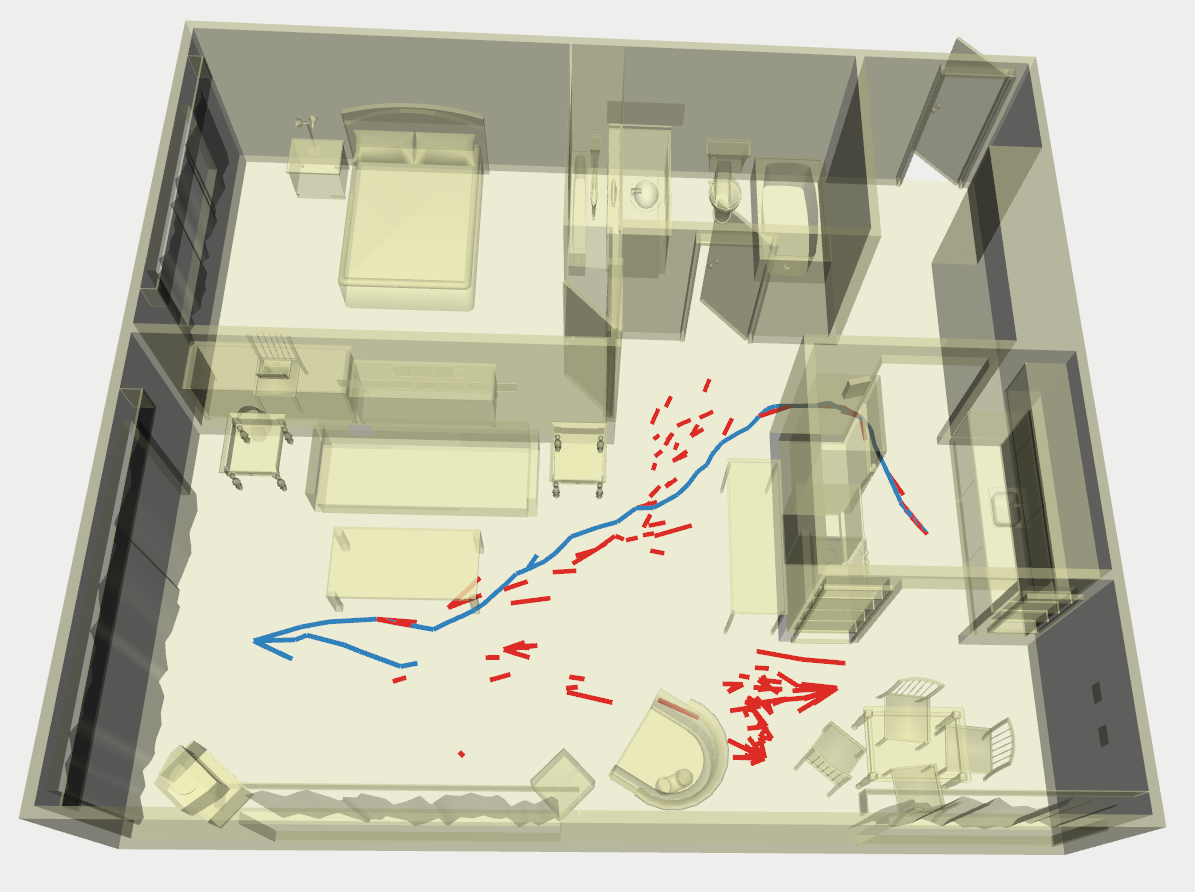}
  \subcaption{}
\end{subfigure}
\caption{(a) Edge priors: darker edges have higher prior. (d) Solution path on the graph. Second and Third columns visualize the search and evaluation by \lazySP{}(top) and \glrastar{}(\eventPathExistence{}) (bottom). (b) and (e): subtree of vertices rewired in the first iteration (12,495 and 1235 resp. at termination). (c) and (f): edges evaluated at termination (63 and 171 resp.).
}
\label{fig:piano} 
\end{figure*}

\begin{figure}[!t]
\centering
\includegraphics[width=\columnwidth]{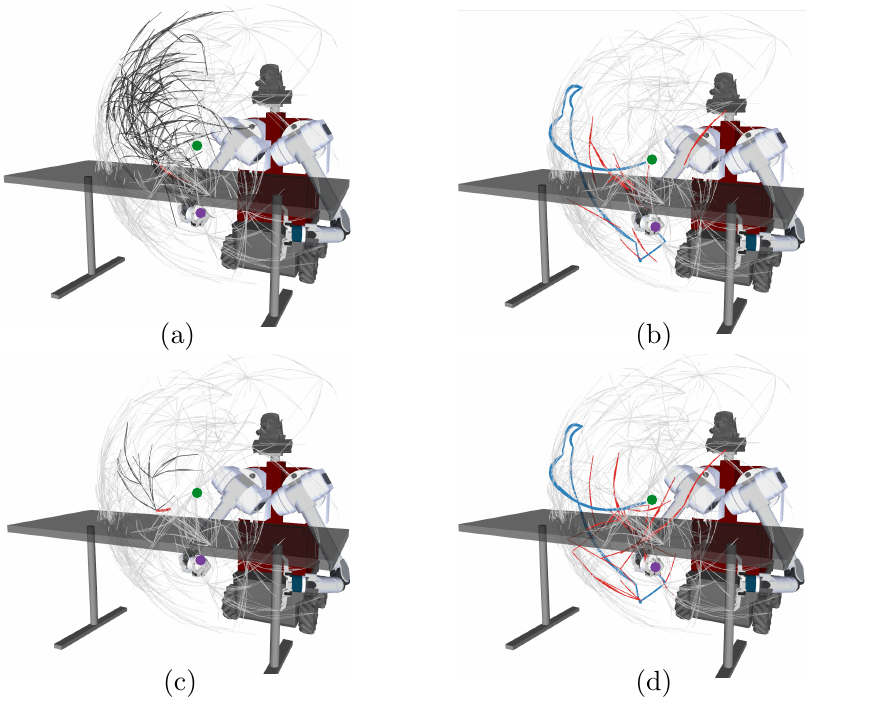}
\caption{Search and evaluation by \lazySP (top) and \glrastar{}(\eventPathExistence) (bottom). (a), (c): subtree of vertices rewired in the first iteration (21,178; 11,342 resp. at termination). (b), (d): edges evaluated (136; 243 resp.).}
\label{fig:herb_scene}
\end{figure}

We consider the Piano Movers' problem in $SE(2)$ from the Apartment scenario in OMPL \cite{sucan2012the-open-motion-planning-library}. For the $\R^7$ environment, we consider two manipulation tasks with a 7-DoF arm \cite{Srinivasa2009} in a cluttered kitchen environment (Fig. \ref{fig:herb_scene_actual}). We used graphs of 8000 vertices and 30,000 vertices for the $SE(2)$ and $\R^7$ problems respectively.

In Table \ref{tab:high_problems}, we report the mean planning times across 100 problems each. We see that \glrastar{}(\eventPathExistence{},~\selectorGreedy{}) outperforms the other algorithms in planning time on all three problems. Additionally, Fig.~\ref{fig:performance:bar_plot} shows a breakdown of the planning time for each of the three events on \emph{HERB Task 2}. \glrastar{} significantly lowers rewiring time while having a minimal increase in evaluation time.

Figures \ref{fig:piano},~\ref{fig:herb_scene} compare the performance of \lazySP{} and \glrastar{} with \selectorGreedy{} selector. They illustrate the savings of \glrastar{} on the Piano Movers' problem (Fig.~\ref{fig:piano}) and on a simplified manipulation scene (Fig.~\ref{fig:herb_scene}). In both cases, \lazySP{} has to rewire a large search tree everytime a path is found to be in collision. \glrastar{}, on the other hand, halts the search as soon as it enters a region of low probability, eliminates the paths and hence drastically minimizes rewiring time at the cost of few additional edge evaluations over \lazySP.

\section{Discussion}
\label{sec:discussion}

We presented a general framework for lazy search (\glrastar). The staple framework interleaves two phases, search and evaluation. 
In the search phase, it extends a lazy shortest-path tree forward without evaluating any edges until an \Event is triggered. 
It then switches to evaluation phase. 
It finds the shortest subpath to a leaf node of the tree and invokes a \Selector to evaluate an edge on it. 
Careful choice of \Event and \Selector allows the balance of search effort with edge evaluation to minimize overall planning time.

The framework, quite expressive, lets us capture a range of lazy search algorithms (Table~\ref{tab:equivalence}). While it draws inspiration from prior work interleaving search and evaluation, such as \lrastar~\cite{Mandalika18}, the key difference lies in our definition of the \Event, which makes the algorithm \emph{adaptive}. This lets us derive new algorithms that are edge optimal while saving on search effort (Theorem~\ref{theorem:optimal}).



In future work, we plan to examine more sophisticated \Selector policies~\cite{SC18} that exploit correlations amongst edges to minimize evaluation cost. We also plan to extend \glrastar to an anytime paradigm; this would let us use heuristics that exploit edge priors to guide the search through regions of high probability~\cite{nielsen2000two}, for significant speed-ups. Finally, we plan to explore problems where multiple lazy estimates of weight functions are available, e.g., in kinodynamic planning, where different relaxations of the boundary value problem can be obtained. We believe \glrastar can interleave search efficiently over multiple resolutions of approximation.





\fontsize{9pt}{10pt}
\selectfont

\bibliographystyle{aaai}
\bibliography{bibliography}

\begin{thebibliography}{}

\bibitem[\protect\citeauthoryear{Bialkowski \bgroup et al\mbox.\egroup
  }{2016}]{BF16}
Bialkowski, J.; Otte, M.~W.; Karaman, S.; and Frazzoli, E.
\newblock 2016.
\newblock Efficient collision checking in sampling-based motion planning via
  safety certificates.
\newblock {\em I. J. Robotics Res.} 35(7):767--796.

\bibitem[\protect\citeauthoryear{Bialkowski, Otte, and
  Frazzoli}{2013}]{bialkowski2013free}
Bialkowski, J.; Otte, M.; and Frazzoli, E.
\newblock 2013.
\newblock {Free-configuration biased sampling for motion planning}.
\newblock In {\em IROS},  1272--1279.
\newblock IEEE.

\bibitem[\protect\citeauthoryear{Bohlin and Kavraki}{2000}]{lazyPRM}
Bohlin, R., and Kavraki, L.~E.
\newblock 2000.
\newblock Path planning using lazy {PRM}.
\newblock In {\em ICRA}, volume~1,  521--528.
\newblock IEEE.

\bibitem[\protect\citeauthoryear{Burns and Brock}{2005}]{burns2005sampling}
Burns, B., and Brock, O.
\newblock 2005.
\newblock Sampling-based motion planning using predictive models.
\newblock In {\em Robotics and Automation, 2005. ICRA 2005. Proceedings of the
  2005 IEEE International Conference on},  3120--3125.
\newblock IEEE.

\bibitem[\protect\citeauthoryear{Choudhury \bgroup et al\mbox.\egroup
  }{2017}]{SC17}
Choudhury, S.; Javdani, S.; Srinivasa, S.; and Scherer, S.
\newblock 2017.
\newblock Near-optimal edge evaluation in explicit generalized binomial graphs.
\newblock In {\em NIPS},  4634--4644.

\bibitem[\protect\citeauthoryear{Choudhury, Dellin, and
  Srinivasa}{2016}]{CDS16}
Choudhury, S.; Dellin, C.~M.; and Srinivasa, S.~S.
\newblock 2016.
\newblock Pareto-optimal search over configuration space beliefs for anytime
  motion planning.
\newblock In {\em IROS},  3742--3749.

\bibitem[\protect\citeauthoryear{Choudhury, Srinivasa, and
  Scherer}{2018}]{SC18}
Choudhury, S.; Srinivasa, S.; and Scherer, S.
\newblock 2018.
\newblock {Bayesian Active Edge Evaluation on Expensive Graphs}.
\newblock In {\em IJCAI},  4890--4897.

\bibitem[\protect\citeauthoryear{Cohen, Phillips, and Likhachev}{2014}]{CPL14}
Cohen, B.~J.; Phillips, M.; and Likhachev, M.
\newblock 2014.
\newblock {Planning Single-arm Manipulations with n-Arm Robots}.
\newblock In {\em RSS}.

\bibitem[\protect\citeauthoryear{Dellin and Srinivasa}{2016}]{DS16}
Dellin, C.~M., and Srinivasa, S.~S.
\newblock 2016.
\newblock A unifying formalism for shortest path problems with expensive edge
  evaluations via lazy best-first search over paths with edge selectors.
\newblock In {\em ICAPS},  459--467.

\bibitem[\protect\citeauthoryear{Dobson and Bekris}{2014}]{SPARS}
Dobson, A., and Bekris, K.~E.
\newblock 2014.
\newblock Sparse roadmap spanners for asymptotically near-optimal motion
  planning.
\newblock {\em I. J. Robotics Res.} 33(1):18--47.

\bibitem[\protect\citeauthoryear{Gammell, Srinivasa, and Barfoot}{2015}]{GSB15}
Gammell, J.~D.; Srinivasa, S.~S.; and Barfoot, T.~D.
\newblock 2015.
\newblock Batch informed trees ({BIT}*): Sampling-based optimal planning via
  the heuristically guided search of implicit random geometric graphs.
\newblock In {\em ICRA},  3067--3074.

\bibitem[\protect\citeauthoryear{Haghtalab \bgroup et al\mbox.\egroup
  }{2018}]{haghtalab2017provable}
Haghtalab, N.; Mackenzie, S.; Procaccia, A.~D.; Salzman, O.; and Srinivasa,
  S.~S.
\newblock 2018.
\newblock {The Provable Virtue of Laziness in Motion Planning}.
\newblock In {\em ICAPS},  106--113.

\bibitem[\protect\citeauthoryear{Halton}{1964}]{Halton64}
Halton, J.~H.
\newblock 1964.
\newblock Algorithm 247: Radical-inverse quasi-random point sequence.
\newblock {\em Commun. ACM} 7(12):701--702.

\bibitem[\protect\citeauthoryear{Hart, Nilsson, and Raphael}{1968}]{HNR68}
Hart, P.~E.; Nilsson, N.~J.; and Raphael, B.
\newblock 1968.
\newblock A formal basis for the heuristic determination of minimum cost paths.
\newblock {\em IEEE Transactions on Systems Science and Cybernetics}
  4(2):100--107.

\bibitem[\protect\citeauthoryear{Hauser}{2015}]{hauser15lazy}
Hauser, K.
\newblock 2015.
\newblock Lazy collision checking in asymptotically-optimal motion planning.
\newblock In {\em ICRA},  2951--2957.

\bibitem[\protect\citeauthoryear{Huh and Lee}{2016}]{huh2016learning}
Huh, J., and Lee, D.~D.
\newblock 2016.
\newblock Learning high-dimensional mixture models for fast collision detection
  in rapidly-exploring random trees.
\newblock In {\em ICRA}.

\bibitem[\protect\citeauthoryear{Janson \bgroup et al\mbox.\egroup
  }{2015}]{JSCP15}
Janson, L.; Schmerling, E.; Clark, A.~A.; and Pavone, M.
\newblock 2015.
\newblock Fast marching tree: {A} fast marching sampling-based method for
  optimal motion planning in many dimensions.
\newblock {\em I. J. Robotics Res.} 34(7):883--921.

\bibitem[\protect\citeauthoryear{Karaman and Frazzoli}{2011}]{KF11}
Karaman, S., and Frazzoli, E.
\newblock 2011.
\newblock Sampling-based algorithms for optimal motion planning.
\newblock {\em I. J. Robotics Res.} 30(7):846--894.

\bibitem[\protect\citeauthoryear{Kavraki \bgroup et al\mbox.\egroup
  }{1996}]{kavraki96prm}
Kavraki, L.~E.; Svestka, P.; Latombe, J.-C.; and Overmars, M.~H.
\newblock 1996.
\newblock Probabilistic roadmaps for path planning in high-dimensional
  configuration spaces.
\newblock {\em {IEEE} Trans. Robotics and Automation} 12(4):566--580.

\bibitem[\protect\citeauthoryear{Kim, Kwon, and Yoon}{2018}]{kim2018adaptive}
Kim, D.; Kwon, Y.; and Yoon, S.-e.
\newblock 2018.
\newblock Adaptive lazy collision checking for optimal sampling-based motion
  planning.
\newblock In {\em UR},  320--327.
\newblock IEEE.

\bibitem[\protect\citeauthoryear{Koenig and Sun}{2009}]{koenig2009}
Koenig, S., and Sun, X.
\newblock 2009.
\newblock Comparing real-time and incremental heuristic search for real-time
  situated agents.
\newblock {\em Autonomous Agents and Multi-Agent Systems} 18(3):313--341.

\bibitem[\protect\citeauthoryear{Koenig, Likhachev, and Furcy}{2004}]{KLF04}
Koenig, S.; Likhachev, M.; and Furcy, D.
\newblock 2004.
\newblock Lifelong planning {A}*.
\newblock {\em Artif. Intell.} 155(1-2):93--146.

\bibitem[\protect\citeauthoryear{Korf}{1985}]{korf1985}
Korf, R.~E.
\newblock 1985.
\newblock Depth-first iterative-deepening: An optimal admissible tree search.
\newblock {\em Artificial Intelligence} 27(1):97 -- 109.

\bibitem[\protect\citeauthoryear{Korf}{1990}]{KORF1990189}
Korf, R.~E.
\newblock 1990.
\newblock Real-time heuristic search.
\newblock {\em Artificial Intelligence} 42(2):189 -- 211.

\bibitem[\protect\citeauthoryear{LaValle}{2006}]{L06}
LaValle, S.~M.
\newblock 2006.
\newblock {\em Planning Algorithms}.
\newblock Cambridge University Press.

\bibitem[\protect\citeauthoryear{Likhachev, Gordon, and
  Thrun}{2004}]{likhachev2004ara}
Likhachev, M.; Gordon, G.~J.; and Thrun, S.
\newblock 2004.
\newblock {ARA*: Anytime A* with provable bounds on sub-optimality}.
\newblock In {\em Advances in neural information processing systems},
  767--774.

\bibitem[\protect\citeauthoryear{Mandalika, Salzman, and
  Srinivasa}{2018}]{Mandalika18}
Mandalika, A.; Salzman, O.; and Srinivasa, S.
\newblock 2018.
\newblock {Lazy Receding Horizon A* for Efficient Path Planning in Graphs with
  Expensive-to-Evaluate Edges}.
\newblock In {\em ICAPS},  476--484.

\bibitem[\protect\citeauthoryear{Murray \bgroup et al\mbox.\egroup
  }{2016}]{murray2016robot}
Murray, S.; Floyd-Jones, W.; Qi, Y.; Sorin, D.~J.; and Konidaris, G.
\newblock 2016.
\newblock Robot motion planning on a chip.
\newblock In {\em RSS}.

\bibitem[\protect\citeauthoryear{Narayanan and
  Likhachev}{2017}]{narayanan2017heuristic}
Narayanan, V., and Likhachev, M.
\newblock 2017.
\newblock Heuristic search on graphs with existence priors for
  expensive-to-evaluate edges.
\newblock In {\em ICAPS}.

\bibitem[\protect\citeauthoryear{Nielsen and Kavraki}{2000}]{nielsen2000two}
Nielsen, C.~L., and Kavraki, L.~E.
\newblock 2000.
\newblock A 2 level fuzzy prm for manipulation planning.
\newblock In {\em IROS}.

\bibitem[\protect\citeauthoryear{Phillips \bgroup et al\mbox.\egroup
  }{2012}]{Phillips-RSS-12}
Phillips, M.; Cohen, B.; Chitta, S.; and Likhachev, M.
\newblock 2012.
\newblock E-graphs: Bootstrapping planning with experience graphs.
\newblock In {\em Proceedings of Robotics: Science and Systems}.

\bibitem[\protect\citeauthoryear{Salzman and Halperin}{2015}]{SH15}
Salzman, O., and Halperin, D.
\newblock 2015.
\newblock {Asymptotically-optimal Motion Planning using lower bounds on cost}.
\newblock In {\em ICRA},  4167--4172.

\bibitem[\protect\citeauthoryear{Salzman and Halperin}{2016}]{SH16}
Salzman, O., and Halperin, D.
\newblock 2016.
\newblock {Asymptotically Near-Optimal {RRT} for Fast, High-Quality Motion
  Planning}.
\newblock {\em {IEEE} Trans. Robotics} 32(3):473--483.

\bibitem[\protect\citeauthoryear{Srinivasa \bgroup et al\mbox.\egroup
  }{2009}]{Srinivasa2009}
Srinivasa, S.~S.; Ferguson, D.; Helfrich, C.~J.; Berenson, D.; Collet, A.;
  Diankov, R.; Gallagher, G.; Hollinger, G.; Kuffner, J.; and Weghe, M.~V.
\newblock 2009.
\newblock {HERB}: a home exploring robotic butler.
\newblock {\em Autonomous Robots} 28(1):5.

\bibitem[\protect\citeauthoryear{{\c{S}}ucan, Moll, and
  Kavraki}{2012}]{sucan2012the-open-motion-planning-library}
{\c{S}}ucan, I.~A.; Moll, M.; and Kavraki, L.~E.
\newblock 2012.
\newblock The {O}pen {M}otion {P}lanning {L}ibrary.
\newblock {\em {IEEE} Robotics \& Automation Magazine} 19(4):72--82.
\newblock \url{http://ompl.kavrakilab.org}.

\bibitem[\protect\citeauthoryear{Yoshizumi, Miura, and
  Ishida}{2000}]{partialA*}
Yoshizumi, T.; Miura, T.; and Ishida, T.
\newblock 2000.
\newblock A* with partial expansion for large branching factor problems.
\newblock In {\em AAAI},  923--929.

\end{thebibliography}

\end{document}